\documentclass{article}

\usepackage{iclr2021_conference,times}
\usepackage[utf8]{inputenc} 
\usepackage[T1]{fontenc}    
\usepackage{hyperref}       
\usepackage{url}            
\usepackage{booktabs}       
\usepackage{amsfonts}       
\usepackage{nicefrac}       
\usepackage{microtype}      
\usepackage{xcolor}
\usepackage{graphicx}
\usepackage{cuted}
\usepackage{capt-of}
\usepackage{tablefootnote}
\usepackage[font=small]{caption} 
\usepackage{multirow} 
\usepackage[ruled,vlined]{algorithm2e}
\usepackage{amsmath, amsthm, amssymb}
\usepackage{subfig}
\usepackage{caption}
\usepackage{wrapfig}
\usepackage{mathtools, textcomp}
\usepackage[normalem]{ulem}

\newcommand{\fig}[1]{Figure~\ref{#1}}
\newcommand{\figref}[1]{Figure~\ref{#1}}

\newcommand{\xx}{\mathbf{x}}
\newcommand{\yy}{\boldsymbol{y}}

\newcommand{\ww}{\mathbf{w}}

\newcommand{\ignore}[1]{}
\newcommand{\secref}[1]{Section~\ref{#1}}
\newtheorem{prop}{Proposition}

\newcommand{\LL}{\mathcal{L}}
\newcommand{\metaVal}{{auxiliary set}}
\newcommand{\GuideNet}{{auxiliary network}}

\newcommand{\AuxiLearn}{AuxiLearn}

\newcommand\revision[1]{\textcolor{black}{#1}}
\newcommand\iclrrevision[1]{\textcolor{black}{#1}}

\iclrfinalcopy

\title{Auxiliary Learning by Implicit \\Differentiation}

\author{%
	Aviv Navon\thanks{Equal contributor}\\
	Bar-Ilan University, Israel\\
	\texttt{aviv.navon@biu.ac.il} \\
	\And
	Idan Achituve$^*$\\
	Bar-Ilan University, Israel\\
	\texttt{idan.achituve@biu.ac.il} \\
	\And
	Haggai Maron\\
	NVIDIA, Israel\\
	\texttt{hmaron@nvidia.com} \\
	\AND
	Gal Chechik\thanks{Equal contributor}\\
	Bar-Ilan University, Israel\\
	NVIDIA, Israel\\
	\texttt{gal.chechik@biu.ac.il} \\
	\And
	Ethan Fetaya$^\dagger$\\
	Bar-Ilan University, Israel\\
	\texttt{ethan.fetaya@biu.ac.il} \\
}

\begin{document}
\maketitle


\begin{abstract}
Training neural networks with auxiliary tasks is a common practice for improving the performance on a main task of interest.
Two main challenges arise in this multi-task learning setting: (i) designing useful auxiliary tasks; and (ii) combining auxiliary tasks into a single coherent loss. Here, we propose a novel framework, \textit{\AuxiLearn{}}, that targets both challenges based on implicit differentiation. First, when useful auxiliaries are known, we propose learning a network that combines all losses into a single coherent objective function. This network can learn \textit{non-linear} interactions between tasks. Second, when no useful auxiliary task is known, we describe how to learn a network that generates a meaningful, novel auxiliary task. We evaluate \AuxiLearn{} in a series of tasks and domains, including image segmentation and learning with attributes in the low data regime, and find that it consistently outperforms competing methods.

\end{abstract}

\section{Introduction}
\label{sec:introduction}
The performance of deep neural networks can significantly improve by training the main task of interest with additional auxiliary tasks
\citep{Goyal_2019_ICCV, jaderberg2016reinforcement, mirowski2019learning}.
For example, learning to segment an image into objects can be more accurate when the model is simultaneously trained to predict other properties of the image like pixel depth or 3D structure  \citep{standley2019tasks}. \revision{In the low data regime, models trained with the main task only are prone to overfit and  generalize poorly to unseen data \citep{vinyals2016matching}. In this case, the benefits of learning with multiple tasks are amplified \citep{zhang2017survey}.} Training with auxiliary tasks adds an inductive bias that pushes learned models to capture meaningful representations and avoid overfitting to spurious correlations.

In some domains, it may be easy to design beneficial auxiliary tasks and collect supervised data. For example, numerous tasks were proposed for self-supervised learning in image classification, including masking \citep{doersch2015unsupervised}, rotation \citep{gidaris2018unsupervised} and patch shuffling \citep{doersch2017multi, noroozi2016unsupervised}. In these cases, it is not clear what would be the best way to combine all auxiliary tasks into a single loss~\citep{doersch2017multi}. The common practice is to compute a weighted combination of pretext losses by tuning the weights of individual losses using hyperparameter grid search. This approach, however, limits the potential of learning with auxiliary tasks because the run time of grid search grows exponentially with the number of tasks.

In other domains, 
obtaining good auxiliaries in the first place may be challenging or may require expert knowledge. 
For example, for point cloud classification, few self-supervised tasks have been proposed; however, their benefits so far are limited  \citep{achituve2020self,hassani2019unsupervised, sauder2019self, tang2020improving}. For these cases, it would be beneficial to automate the process of generating auxiliary tasks without domain expertise. 

Our work takes a step forward in automating the use and design of auxiliary learning tasks. We name our approach \textit{\AuxiLearn{}}. \AuxiLearn{} leverages recent progress made in implicit differentiation for optimizing hyperparameters~\citep{liao18, lorraine2019optimizing}. 
We demonstrate the effectiveness of \AuxiLearn{} in two types of problems. First, in \textbf{combining auxiliaries}, for cases where auxiliary tasks are predefined. We describe how to train a deep neural network (NN) on top of auxiliary losses and combine them non-linearly into a unified loss. For instance, we combine per-pixel losses in image segmentation tasks using a convolutional NN (CNN). Second, \textbf{designing auxiliaries}, for cases where predefined auxiliary tasks are not available. We present an approach for learning such auxiliary tasks without domain knowledge and from input data alone. 
This is achieved by training an auxiliary network to generate auxiliary labels while training another, primary network to learn both the original task and the auxiliary task. One important distinction from previous works, such as \citep{kendall2018multi, liu2019self}, is that we do not optimize the auxiliary parameters using the training loss but rather on a separate (small) \emph{\metaVal{}}, allocated from the training data. This is a key difference since the goal of auxiliary learning is to improve generalization rather than help optimization on the training data.

To validate our proposed solution, we extensively evaluate AuxiLearn in several tasks in the low-data regime. In this regime, the models suffer from severe overfitting and auxiliary learning can provide the largest benefits. Our results demonstrate that using \AuxiLearn{} leads to improved loss functions and auxiliary tasks, in terms of the performance of the resulting model on the main task. We complement our experimental section with two interesting theoretical insights regarding our model. The first shows that a relatively simple auxiliary hypothesis class may overfit. The second aims to understand which auxiliaries benefit the main task. 

To summarize, we propose a novel general approach for learning with auxiliaries using implicit differentiation. We make the following novel contributions: 
(a) \revision{We describe a unified approach for combining multiple loss terms and for learning novel auxiliary tasks from the data alone;}
(b) \revision{We provide a theoretical observation on the capacity of auxiliary learning;} 
(c) \revision{We show that the key quantity for determining beneficial auxiliaries is the Newton update;}
(d) We provide new results on a variety of auxiliary learning tasks \revision{with a focus on the low data regime}.
We conclude that implicit differentiation can play a significant role in automating the design of auxiliary learning setups.

\section{Related work}
\label{sec:related_work}
\noindent\textbf{Learning with multiple tasks.}\quad \revision{Multitask Learning (MTL) aims at simultaneously solving} multiple learning problems while sharing information across tasks. In some cases, MTL benefits the optimization process and improves task-specific generalization performance compared to single-task learning ~\citep{standley2019tasks}. In contrast to MTL, \revision{auxiliary learning aims at solving }
a single, main task, and the purpose of all other tasks is to facilitate the learning of the primary task. At test time, only the main task is considered. This approach has been successfully applied in multiple domains, including computer vision~\citep{zhang2014facial}, natural language processing~\citep{fan2017transfer, trinh2018learning}, and reinforcement learning~\citep{jaderberg2016reinforcement, lin2019adaptive}.

\noindent\textbf{Dynamic task weighting.}\quad When learning a set of tasks, the task-specific losses are combined into an overall loss. The way individual losses are combined is crucial because MTL-based models are sensitive to the relative weightings of the tasks \citep{kendall2018multi}. A common approach for combining task losses is in a linear fashion. When the number of tasks is small, task weights are commonly tuned with a simple grid search. However,
this approach does not extend to a large number of tasks, or a more complex weighting scheme. Several recent studies proposed scaling task weights using gradient magnitude \citep{chen2017gradnorm}, task uncertainty \citep{kendall2018multi}, or the rate of loss change \citep{liu2019end}.  \cite{sener2018multi} proposed casting the multitask learning problem as a multi-objective optimization. These methods assume that all tasks are equally important, and are less suited for auxiliary learning. \cite{du2018adapting} and \cite{lin2019adaptive} proposed to weight auxiliary losses using gradient similarity. However, these methods do not scale well with the number of auxiliaries and do not take into account interactions between auxiliaries. In contrast, we propose to learn from data how to combine auxiliaries, possibly in a non-linear manner.

\begin{figure}[!t]
    \centering
    \subfloat[Combining losses]{
    \includegraphics[width=0.44\linewidth]{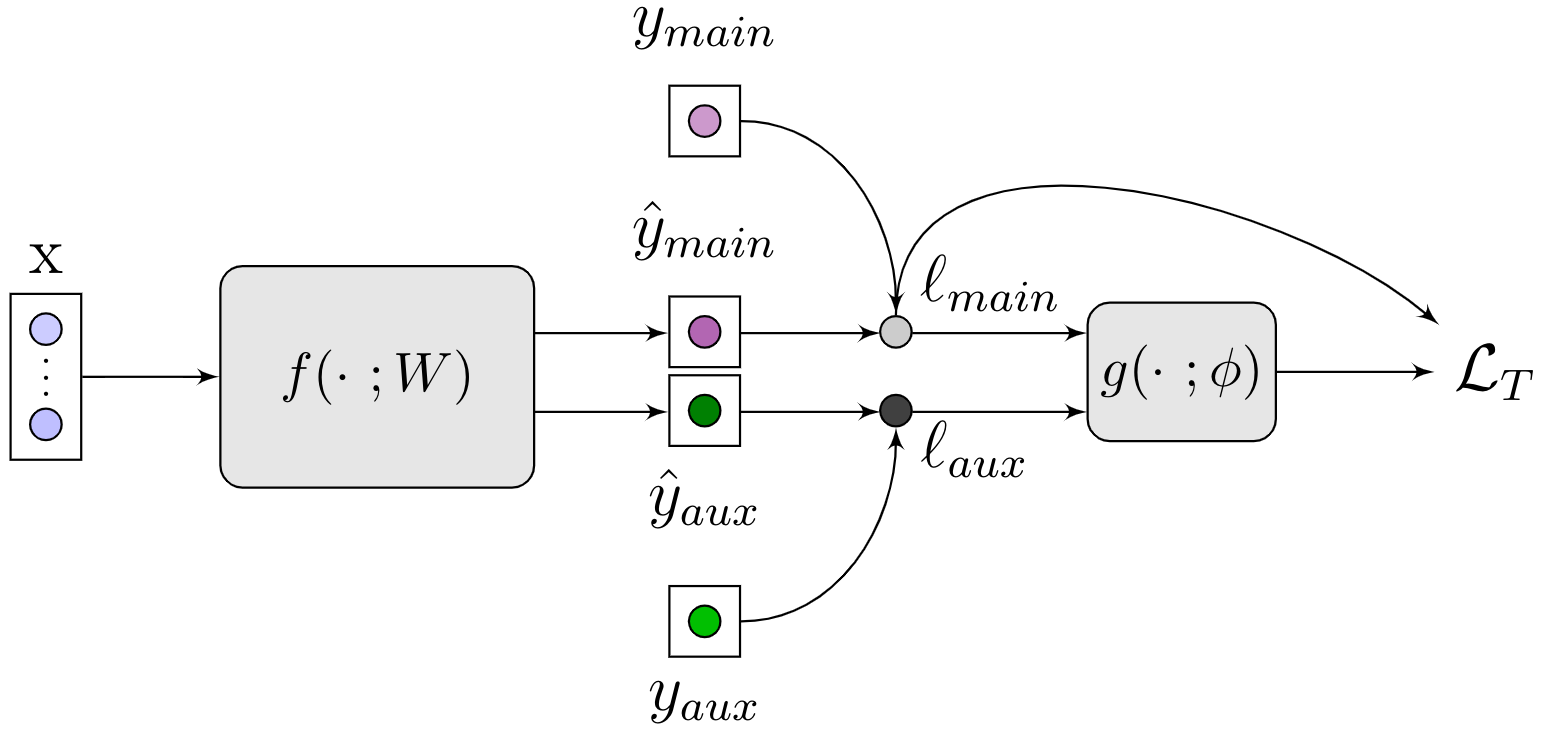}
    \label{fig:combining_losses}}
    \subfloat[Learning a new auxiliary task]{
    \hspace{40pt}
    \includegraphics[width=0.36\linewidth]{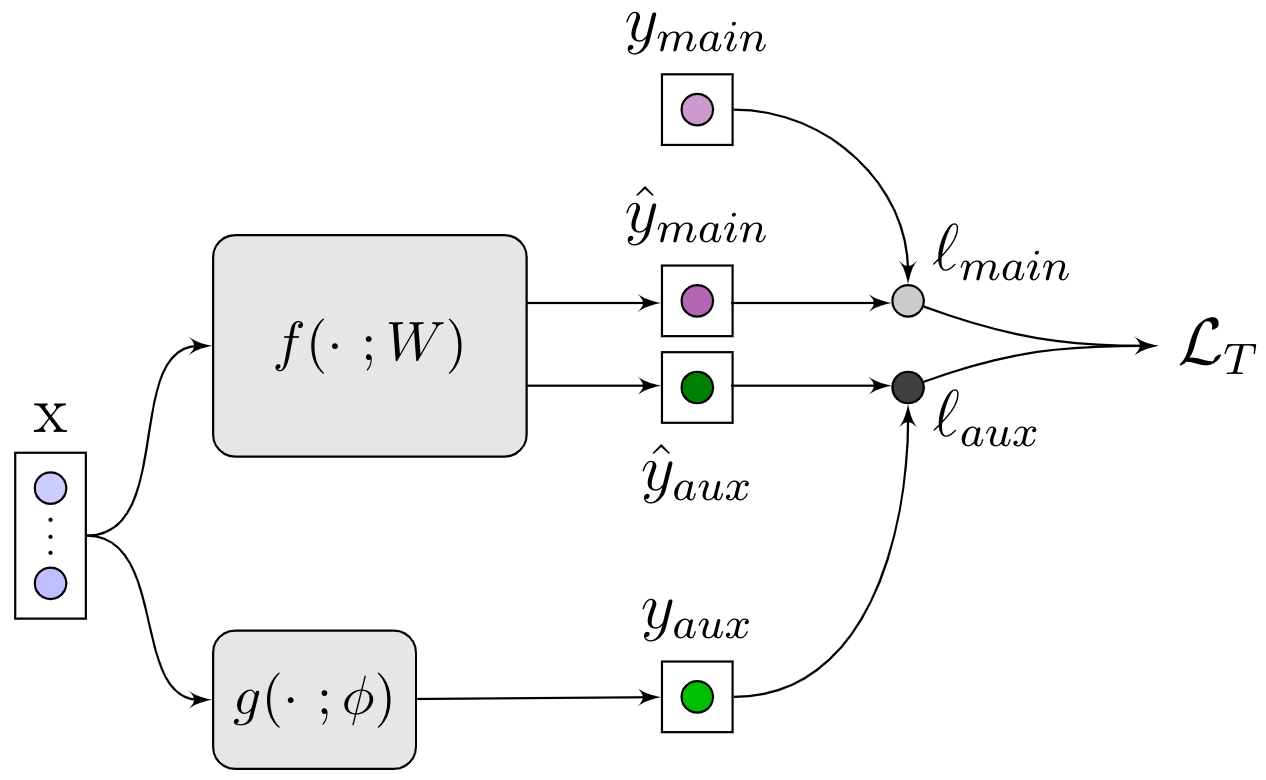}
    \label{fig:learn_new_aux}}
    \caption{
    The \AuxiLearn{} framework. \textbf{(a)} Learning to combine losses into a single coherent loss term. Here, the \GuideNet{} operates over a vector of losses. 
    \textbf{(b)} Generating a novel auxiliary task. Here the \GuideNet{} operates over the input space. In both cases, $g(\cdot~;\phi)$ is optimized using IFT based on $\LL_A$.}
    \label{fig:GuideNet_arch}
\end{figure}

\noindent\textbf{Devising auxiliaries.}\quad Designing an auxiliary task for a given main task is challenging because it may require domain expertise and additional labeling effort. For self-supervised learning (SSL), many approaches have been proposed (see \cite{jing2020self} for a recent survey), but the joint representation learned through SSL may suffer from negative transfer and hurt the main task \citep{standley2019tasks}. \citet{liu2019self} proposed learning a helpful auxiliary in a meta-learning fashion, removing the need for handcrafted auxiliaries. However, their system is optimized for the training data, which may lead to degenerate auxiliaries. To address this issue, an entropy term is introduced to force the auxiliary network to spread the probability mass across classes.

\noindent\textbf{Implicit differentiation based optimization.}\quad Our formulation gives rise to a bi-level optimization problem. Such problems naturally arise in the context of meta-learning \citep{finn2017model, rajeswaran2019meta} and hyperparameter optimization \citep{bengio2000gradient, foo2008efficient, larsen1996design, liao18, lorraine2019optimizing, pedregosa2016hyperparameter}. The Implicit Function Theorem (IFT) is often used for computing gradients of the upper-level function, this operation requires calculating a vector-inverse Hessian product.
However, for modern neural networks, it is infeasible to calculate it explicitly, and an approximation must be devised. \citet{luketina2016scalable} proposed approximating the Hessian with the identity matrix, whereas \citet{foo2008efficient, pedregosa2016hyperparameter,rajeswaran2019meta} used conjugate gradient (CG) to approximate the product. Following \citet{liao18,lorraine2019optimizing}, we use a truncated Neumann series and efficient vector-Jacobian products, as it was empirically shown to be more stable than CG.

\section{Our Method}
\label{sec:approach}
\vspace{-5pt}
We now describe the general \AuxiLearn{} framework for learning with auxiliary tasks. For that purpose, we use two networks, a primary network that is optimized on all tasks and an auxiliary network that is optimized on the main task only. First, we introduce our notations and formulate the general objective. Then, we describe two instances of this framework: combining auxiliaries and learning new auxiliaries. Finally, we present our optimization approach for both instances.

\subsection{Problem definition}
Let $\{(\xx_i^t,\yy_i^t) \}_i$ be the training set and $\{(\xx_i^a,\yy_i^a) \}_i$ be a distinct independent set which we term \textit{\metaVal}.
Let $f(\cdot~;W)$ denote the primary network, and let $g(\cdot~;\phi)$ denote the \GuideNet{}. Here, $W$ are the parameters of the model optimized on the training set, and $\phi$ are the auxiliary parameters trained on the \metaVal. The training loss is defined as:
\begin{equation}
    \label{eq:train_objective}
    \LL_T=\LL_T(W,\phi)=\sum_i\ell_{main}(\xx_i^t,\yy_i^t;W)+h(\xx_i^t,\yy_i^t,W;\phi),
\end{equation}
where $\ell_{main}$ denotes the loss of the main task and $h$ is the overall auxiliary loss, controlled by $\phi$. In Sections \ref{sec:combine_losses} \& \ref{sec:learning_new_auxiliries} we will describe two instances of $h$. We note that $h$ has access to both $W$ and $\phi$. The loss on the \metaVal{} is defined as $\LL_{A}=\sum_i\ell_{main}(\xx_i^a,\yy_i^a;W)$, since we are interested in the generalization performance of the main task.

We wish to find auxiliary parameters ($\phi$) such that the primary parameters ($W$), trained with the combined objective, generalize well. More formally, we seek
\begin{equation}
    \label{eq:bi_level}
    \phi^*=\arg\min_{\phi}\LL_A(W^*(\phi)),\quad \text{s.t.} \quad W^*(\phi)=\arg\min_W\LL_T(W,\phi).
\end{equation}

\subsection{Learning to combine auxiliary tasks}\label{sec:combine_losses}

Suppose we are given $K$ auxiliary tasks, usually designed using expert domain knowledge. We wish to learn how to optimally leverage these auxiliaries by learning to combine their corresponding losses. Let $\boldsymbol{\ell}(\xx,\yy;W)=(\ell_{main}(\xx,y^{main};W), \ell_1(\xx,y^{1};W), ...,\ell_K(\xx,y^{K};W))$ denote a loss vector. We wish to learn an \GuideNet{} $g:\mathbb{R}^{K+1}\to\mathbb{R}$ over the losses that will be added to $\ell_{main}$ in order to output the training loss $\LL_T=\ell_{main}+g(\boldsymbol{\ell};\phi)$. Here, $h$ from Eq.~\eqref{eq:train_objective} is given by  $h(\cdot~;\phi)=g(\boldsymbol{\ell};\phi)$.

Typically, $g(\boldsymbol{\ell};\phi)$ is chosen to be a linear combination of the losses: $g(\boldsymbol{\ell};\phi)=\sum_j \phi_j\ell_j$, with positive weights $\phi_j\ge0$ that are tuned using a grid search. However, this method can only scale to a few auxiliaries, as the run time of grid search is exponential in the number of tasks. Our method can handle a large number of auxiliaries and easily extends to a more flexible formulation in which $g$ parametrized by a deep NN. This general form allows us to capture complex interactions between tasks, and learn non-linear combinations of losses. See \figref{fig:combining_losses} for illustration.

One way to view a non-linear combination of losses is as an adaptive linear weighting, where losses have a different set of weights for each datum. If the loss at point $\xx$ is $\ell_{main}(\xx,y^{main})+g(\boldsymbol{\ell}(\xx,\yy))$, then the gradients are $\nabla_{W}\ell_{main}(\xx,y^{main})+\sum_j \frac{\partial  g}{\partial\ell_j}\nabla_{W}\ell_j(\xx,y^j)$. This is equivalent to an adaptive loss where the loss of datum $\xx$ is 
$\ell_{main}+\sum_j \alpha_{j,\xx}\ell_j$ and, $\alpha_{j,\xx}=\frac{\partial  g}{\partial\ell_j}$. This observation connects our approach to other studies that assign adaptive loss weighs (e.g., \citet{du2018adapting,liu2019end}).

\textbf{Convolutional loss network.}\quad In certain problems there exists a spatial relation among losses. For example, semantic segmentation and depth estimation for images. A common approach is to average the losses over all locations. In contrast, \AuxiLearn{} can leverage this spatial relation for creating a \emph{loss-image} in which each task forms a channel of pixel-losses induced by the task. We then parametrize $g$ as a CNN that acts on this loss-image. This yields a spatial-aware loss function that captures interactions between task losses. See an example of a loss image in Figure~\ref{fig:nyu_loss_grad} 

\textbf{Monotonicity.}\quad It is common to parametrize the function $g(\boldsymbol{\ell};\phi)$ as a linear combination with non-negative weights. Under this parameterization, $g$ is a monotonic non-decreasing function of the losses. A natural question that arises is whether we should generalize this behavior and constrain $g(\boldsymbol{\ell};\phi)$ to be non-decreasing w.r.t. the input losses as well? Empirically, we found that training with monotonic non-decreasing networks tends to be more stable and has a better or equivalent performance. We impose monotonicity during training with negative weights clipping. See Appendix \ref{app:monotonicity} for a detailed discussion and empirical comparison to non-monotonic networks.

\subsection{Learning new auxiliary tasks} \label{sec:learning_new_auxiliries}
The previous subsection focused on situations where auxiliary tasks are given. In many cases, however, no useful auxiliary tasks are known in advance, and we are only presented with the main task. 
We now describe how to use AuxiLearn in such cases.
The intuition is simple: We wish to learn an auxiliary task that pushes the representation of the primary network to generalize better on the main task, as measured using the auxiliary set. We do so in a student-teacher manner: 
an auxiliary ``teacher'' network produces labels for the primary network (the ``student'') which tries to predict these labels as an auxiliary task. Both networks are trained jointly.

\revision{
More specifically, for auxiliary classification, we learn a soft labeling function $g(\xx;\phi)$ which produces pseudo labels $y_{aux}$ for input samples $\xx$ . These labels are then provided to the main network $f(\xx;W)$ for training (see~\figref{fig:learn_new_aux}). 
}  
During training, the \revision{primary network $f(\xx;W)$} outputs two predictions, $\hat{y}_{main}$ for the main task and $\hat{y}_{aux}$ for the auxiliary task. 
We then compute the full training loss $\LL_T=\ell_{main}(\hat{y}_{main}, y_{main}) + \ell_{aux}(\hat{y}_{aux}, y_{aux})$ to update $W$. 
Here, the $h$ component of $\LL_T$ in Eq.~\eqref{eq:train_objective} is given by $h(\cdot~;\phi)=\ell_{aux}(f(\xx_i^t;W),g(\xx_i^t;\phi))$. As before, we update $\phi$ using the auxiliary set with the loss  $\LL_A=\ell_{main}$.  
Intuitively, the teacher auxiliary network $g$ is rewarded when it provides labels to the student that help it succeed in the main task, as measured using $\LL_A$. 



\subsection{Optimizing auxiliary parameters} \label{sec:unified_opt}
We now return to the bi-level optimization problem in Eq. \eqref{eq:bi_level} and present the optimizing method for $\phi$.
Solving Eq. \eqref{eq:bi_level} for $\phi$ poses a problem due to the indirect dependence of $\LL_A$ on the auxiliary parameters. To compute the gradients of $\phi$, we need to differentiate through the optimization process over $W$, since $\nabla_\phi\LL_A=\nabla_W\LL_A\cdot\nabla_\phi W^*$. As in \citet{liao18,lorraine2019optimizing}, we use the implicit function theorem (IFT) to evaluate $\nabla_{\phi}W^*$:
\begin{equation}
    \label{eq:phi_grad_0}
    \nabla_{\phi}W^*=-\underbrace{(\nabla^2_W\LL_T)^{-1}}_{|W|\times|W|} \cdot\underbrace{\nabla_\phi\nabla_W\LL_T}_{|W|\times|\phi|}.
\end{equation}
We can leverage the IFT to approximate the gradients of the auxiliary parameters $\phi$:
\begin{equation}
    \label{eq:ift_dw_dphi}
    \nabla_\phi{\LL_A(W^*(\phi))}=-\underbrace{\nabla_W\LL_A}_{1\times |W|} \cdot \underbrace{(\nabla^2_W\LL_T)^{-1}}_{|W|\times|W|} \cdot\underbrace{\nabla_\phi\nabla_W\LL_T}_{|W|\times|\phi|}.
\end{equation}
See Appendix~\ref{app:ift_derivation} for a detailed derivation. 
To compute the vector and Hessian inverse product, we use the algorithm proposed by \citet{lorraine2019optimizing}, which uses Neumann approximation and efficient vector-Jacobian product. We note that accurately computing $\nabla_\phi\LL_A$ by IFT requires finding a point such that $\nabla_W{\LL_T} = 0$. In practice, we only approximate $W^*$, and simultaneously train both $W$ and $\phi$ by altering between optimizing $W$ on $\LL_T$, and optimizing $\phi$ using $\LL_A$. We summarize our method in Alg. \ref{alg:auxilearn} and \ref{alg:hypergrad}. Theoretical considerations regarding our method are given in Appendix~\ref{theoretical_cons}.

\begin{minipage}{0.48\textwidth}
\begin{algorithm}[H]
    \footnotesize
    \SetAlgoLined
    Initialize auxiliary parameters $\phi$ and weights $W$;
    \While{not converged}{
      \For{$k=1,...,N$}{
      $\LL_T=\ell_{main}(\xx,y;W)+h(\xx,y,W;\phi)$\\
      $W  \leftarrow W \mathrel{-} \alpha \nabla_{W} \LL_T\bigm\vert_{\phi, W}$
      }
      $\phi \leftarrow \phi\mathrel{-}\text{Hypergradient}(\LL_A, \LL_T, \phi, W)$ 
     }
     \Return{$W$}
     \label{alg:auxilearn}
     \caption{AuxiLearn}
\end{algorithm}
\end{minipage}
\hfill
\begin{minipage}{0.48\textwidth}
\begin{algorithm}[H]
    \footnotesize
    \SetAlgoLined
     \KwIn{training loss $\LL_T$, auxiliary loss $\LL_A$, a fixed point $(\phi', W^*)$, number of iterations $J$, learning rate $\alpha$}
     $v=p=\nabla_W\LL_A\vert_{(\phi', W^*)}$\\
     \For{$j=1,...,J$}{
        $v\mathrel{-}=\alpha v \cdot \nabla_W\nabla_W\LL_T$\\ 
        $p\mathrel{+}=v$
     }
     \KwRet{$-p\nabla_{\phi}\nabla_W\LL_T\vert_{(\phi', W^*)}$}
     \label{alg:hypergrad}
     \caption{Hypergradient}
\end{algorithm}
\end{minipage}

\section{Analysis} \label{sec:analysis}

\subsection{Complexity of auxiliary hypothesis space}
In our learning setup, an additional \metaVal{} is used for tuning a large set of auxiliary parameters. A natural question arises: could the auxiliary parameters overfit this \metaVal{}? and what is the complexity of the auxiliary hypothesis space $\mathcal{H}_\phi$? Analyzing the complexity of this space is difficult because it is coupled with the hypothesis space $\mathcal{H}_W$  of the main model. One can think of this hypothesis space as a subset of the original model hypothesis space $\mathcal{H}_\phi=\{h_W:\exists\phi\,\text{ s.t. } W=\arg\min_W\LL_T(W,\phi) \}\subset\mathcal{H}_W$. Due to the coupling with $\mathcal{H}_W$ the behavior can be unintuitive. We show that even simple auxiliaries can have infinite VC dimensions. 

\textbf{Example:}\quad Consider the following 1D hypothesis space for binary classification $\mathcal{H}_W=\{\lceil \cos(W x)\rceil, W\in\mathbb{R}\}$, which has  infinite VC-dimension. Let the main loss be the zero-one loss and the auxiliary loss be $h(\phi,W)=(\phi-W)^2$, namely, an  $L_2$ regularization with a learned center. Since the model hypothesis space $\mathcal{H}_W$ has an infinite VC-dimension, there exist training and auxiliary sets of any size that are shattered by $\mathcal{H}_W$. Therefore, for any labeling of the auxiliary and training sets, we can let $\phi=\hat{\phi}$, the parameter that perfectly classifies both sets. We then have that $\hat{\phi}$ is the optimum of the training with this auxiliary loss and we get that $\mathcal{H}_\phi$ also has an infinite VC-dimension. 

This important example 
shows that even seemingly simple-looking auxiliary losses can overfit due to the interaction with the model hypothesis space. Thus, it motivates our use of a separate auxiliary set.

\subsection{Analyzing an auxiliary task effect}
When designing or learning auxiliary tasks, one important question is, what makes an auxiliary task useful? 
Consider the following loss with a single auxiliary task \(\LL_T(W,\phi)=\sum_i\ell_{main}(\xx_i^t,\yy_i^t,W)+\phi\cdot \ell_{aux}(\xx_i^t,\yy_i^t,W)\). Here \(h=\phi\cdot\ell_{aux}\). Assume $\phi=0$ so we optimize $W$ only on the standard main task loss. We can now check if $\frac{d\LL_A}{d\phi}|_{\phi=0}>0$, namely would it help to add this auxiliary task?

\begin{prop}
Let  $\LL_T(W,\phi)=\sum_i\ell_{main}(\xx_i^t,\yy_i^t,W)+\phi\cdot \ell_{aux}(\xx_i^t,\yy_i^t,W)$. Suppose that $\phi=0$ and that the main task was trained until convergence. We have 
\begin{equation}
    \frac{d\LL_A(W^*(\phi))}{d\phi}\Big\vert_{\phi=0}=-\langle\nabla_W\LL_A^T, \nabla^2_W\LL_T^{-1} \nabla_W\LL_T \rangle, 
\end{equation}
i.e. the gradient with respect to the auxiliary weight is the inner product between the Newton methods update and the gradient of the loss on the auxiliary set.
\end{prop}
\begin{proof}
    In the general case, the following holds  $\frac{d\LL_A}{d\phi}=-\nabla_W\LL_A(\nabla^2_W\LL_T)^{-1}\nabla_{\phi}\nabla_W\LL_T$. For a linear combination, we have $\nabla_{\phi}\nabla_W\LL_T=\sum_i\nabla_W\ell_{aux}(\xx_i^t,\yy_i^t)$. Since $W$ is optimized till convergence of the main task we obtain $\nabla_{\phi}\nabla_{W}\LL_T=\nabla_W\LL_T$.
\end{proof}
This simple result shows that the key quantity to observe is the Newton update, rather than the gradient which is often used~\citep{lin2019adaptive,du2018adapting}. Intuitively, the Newton update is the important quantity because if $\Delta\phi$ is small then we are almost at the optimum. Then, due to quadratic convergence, a single Newton step is sufficient for approximately converging to the new optimum.

\section{Experiments}
\vspace{-5pt}
\label{sec:experiments}
We evaluate the \AuxiLearn{} framework in a series of tasks of two types: combining given auxiliary tasks into a unified loss (Sections \ref{sec:illustrative} - \ref{sec:NYU}), and generating a new auxiliary task (Section \ref{sec:learning_cls_exp}). \revision{Further experiments and analysis of both modules are given in Appendix~\ref{sec:additional_exp}.} 
Throughout all experiments, we use an extra data split for the \metaVal{}. Hence, we use four data sets: training set, validation set, test set, and \metaVal{}. The samples for the \metaVal{} are pre-allocated from the training set. For a fair comparison, these samples are used as part of the training set by all competing methods. Effectively, this means we have a slightly smaller training set for optimizing the parameters $W$ of the primary network. In all experiments, we report the mean performance (e.g., accuracy) along with the Standard Error of the Mean (SEM). 
\revision{Full implementation details of all experiments are given in Appendix~\ref{training_details}.}
Our code is available at \textcolor{magenta}{\url{https://github.com/AvivNavon/AuxiLearn}}.


\textbf{Model variants.}\quad 
For learning to combine losses, we evaluated the following variants of auxiliary networks: \textbf{(1) Linear}: A convex linear combination between the loss terms; \textbf{(2) Linear neural network (Deep linear)}: A deep fully-connected NN with linear activations; \textbf{(3) Nonlinear}: A standard feed-forward NN over the loss terms. For 
Section~\ref{sec:NYU} only \textbf{(4) ConvNet}: A CNN over the loss-images. The expressive power of the deep linear network is equivalent to that of a 1-layer linear network; However, from an optimization perspective, it was shown that the over-parameterization introduced by the network's depth could stabilize and accelerate convergence \citep{arora2018optimization, saxe2014exact}. All variants are constrained to represent only monotone non-decreasing functions.
\subsection{An illustrative example} \label{sec:illustrative}

\begin{wrapfigure}[10]{r}{0.6\textwidth}
    \vspace{-20pt}
    \small
    \text{
        \hspace{25pt} (a) main task
        \hspace{27pt} (b) $t=0$ 
        \hspace{27pt} (c) $t=T$
        \hspace{80pt}
        }
    \centering
    \includegraphics[width=.95\linewidth]{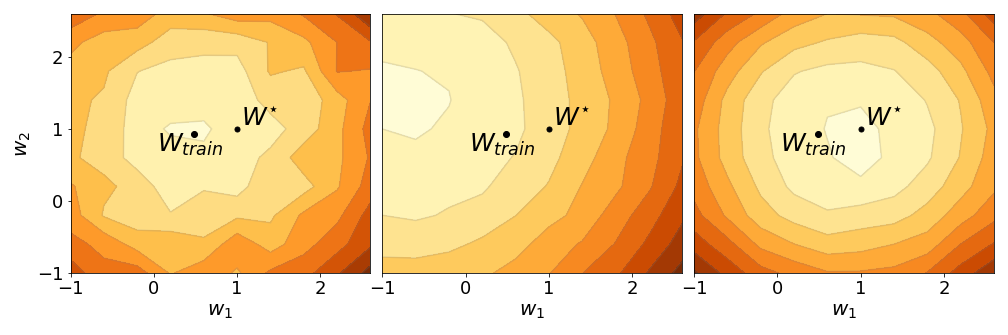}
    \caption{Loss landscape generated by the \GuideNet{}. Darker is higher. See text for details.}
    \label{fig:loss_landscape}
\end{wrapfigure}
We first present an illustrative example of how \AuxiLearn{} changes the loss landscape and helps generalization in the presence of label noise and harmful tasks. Consider a regression problem with $y_{main}=\ww^{\star T} \xx+\epsilon_0$ and two auxiliary tasks. The first auxiliary is helpful, $y_1=\ww^{\star T} \xx + \epsilon_1$, whereas the second auxiliary is harmful $y_{2}=\tilde{\ww}^T \xx + \epsilon_{2}$, $\tilde{\ww}\ne\ww^\star$. We let $\epsilon_0\sim \mathcal{N}(0, \sigma_{main}^2)$ and $\epsilon_1,\epsilon_2\sim \mathcal{N}(0, \sigma_{aux}^2)$, with $\sigma_{main}^2 > \sigma_{aux}^2$. We optimize a linear model with weights $\ww \in \mathbb{R}^2$ that are shared across tasks, i.e., no task-specific parameters. We set $\ww^{\star}=(1, 1)^T$ and $\tilde{\ww}=(2, -4)^T$. We train an \GuideNet{} to output linear task weights and observe the changes to the loss landscape in \figref{fig:loss_landscape}. The left plot shows the loss landscape for the main task, with a training set optimal solution $\ww_{train}$. Note that $\ww_{train}\neq\ww^*$ due to the noise in the training data. The loss landscape of the weighted train loss at the beginning ($t=0$) and the end ($t=T$) of training is shown in the middle and right plots, respectively. Note how AuxiLearn learns to ignore the harmful auxiliary and use the helpful one to find a better solution by changing the loss landscape. In Appendix~\ref{sec:noisy_aux} 
we 
show that the auxiliary task weight is inversely proportional to the label noise.

\begin{figure}[!t]
    \text{
    \small
        \hspace{10pt} (a) image 
        \hspace{15pt} (b) GT labels
        \hspace{8pt} (c) aux. loss
        \hspace{7pt} (d) main loss 
        \hspace{5pt} (e) pix. weight
        \hspace{3pt}
        }
    \centering
    \includegraphics[width=0.75\linewidth]{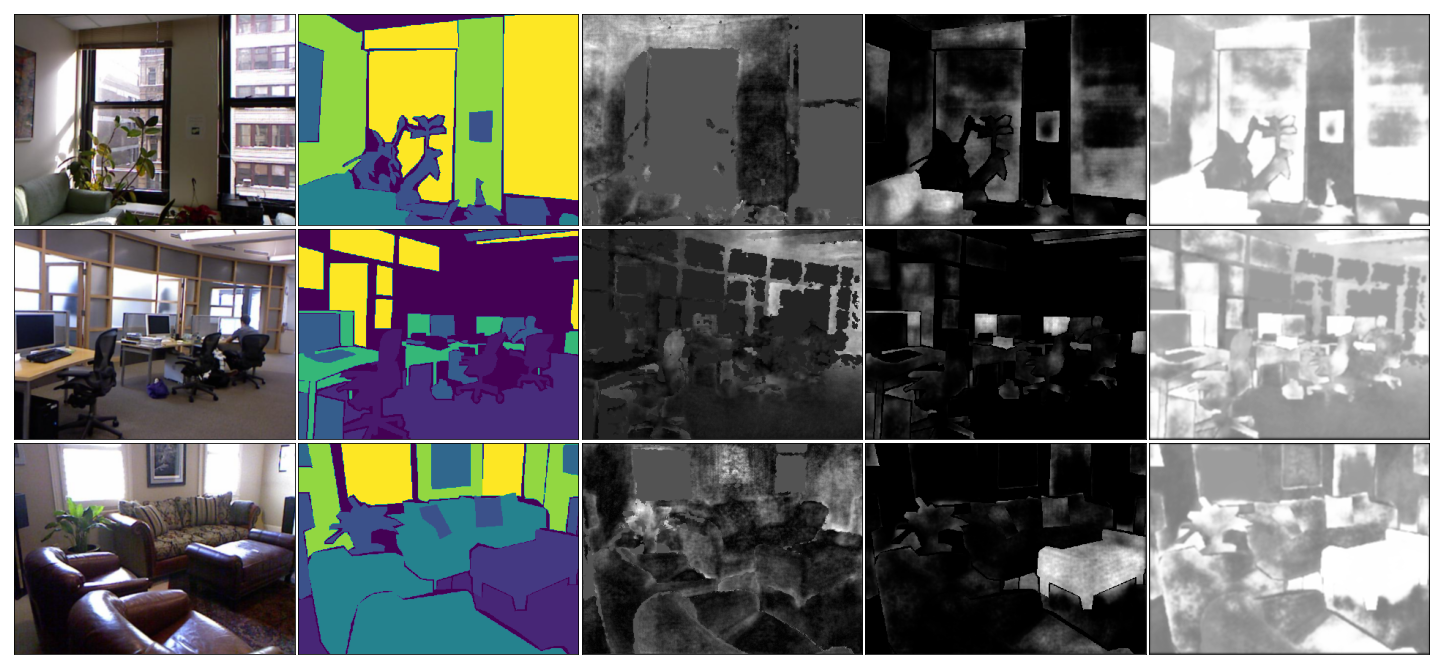}
    \caption{\textit{Loss images}  on test examples from NYUv2: \textbf{(a)} original image; \textbf{(b)} semantic segmentation ground truth; \textbf{(c)} auxiliaries loss; \textbf{(d)} segmentation (main task) loss; \textbf{(e)} adaptive pixel-wise weight $ \sum_j\partial \LL_T/\partial \ell_j$.}
    \label{fig:nyu_loss_grad}
\end{figure}

\subsection{Fine-grained classification with many auxiliary tasks} \label{sec:cub_exp}
In fine-grained visual classification tasks, annotators should have domain expertise, making data labeling challenging and potentially expensive (e.g., in the medical domain). In some cases, however, non-experts can annotate visual attributes that are informative about the main task. As an example, consider the case of recognizing bird species, which would require an ornithologist, yet a layman can describe the head color or bill shape of a bird. These features naturally form auxiliary tasks, which can be leveraged for training jointly with the main task of bird classification.

\begin{wraptable}[14]{r}{0.55\linewidth}
\vspace{-11pt}
\setlength{\tabcolsep}{3pt}
\scriptsize
\caption{ Test classification accuracy on CUB 200-2011 dataset, averaged over three runs ($\pm$ SEM).}
\centering
\begin{tabular}{l c c c c c}
\toprule
 & \multicolumn{2}{c}{5-shot} & & \multicolumn{2}{c}{10-shot} \\
 \cmidrule{2-3}\cmidrule{5-6}
 & Top 1 & Top 3 & & Top 1 & Top 3 \\ 
\midrule
STL & 35.50 $\pm$ 0.7 & 54.79 $\pm$ 0.7 & & 54.79 $\pm$ 0.3 & 74.00 $\pm$ 0.1 \\
Equal & 41.47 $\pm$ 0.4 & 62.62 $\pm$ 0.4 & & 55.36 $\pm$ 0.3 & 75.51 $\pm$ 0.4 \\
\hline
Uncertainty & 35.22 $\pm$ 0.3 & 54.99 $\pm$ 0.7 & & 53.75 $\pm$ 0.6 & 73.25 $\pm$ 0.3 \\
DWA & 41.82 $\pm$ 0.1 & 62.91 $\pm$ 0.4 & & 54.90 $\pm$ 0.3 & 75.74 $\pm$ 0.3 \\
GradNorm & 41.49 $\pm$ 0.4 & 63.12 $\pm$ 0.4  & & 55.23 $\pm$ 0.1 &75.62 $\pm$ 0.3 \\
GCS & 42.57 $\pm$ 0.7 & 62.60 $\pm$ 0.1 & & 55.65 $\pm$ 0.2 & 75.71 $\pm$ 0.1 \\
\midrule
\textbf{\AuxiLearn{}} &  &  &  &  \\
Linear & 41.71 $\pm$ 0.4 & 63.73 $\pm$ 0.6 & & 54.77 $\pm$ 0.2 & 75.51 $\pm$ 0.7 \\
Deep Linear & 45.84 $\pm$ 0.3 & 66.21 $\pm$ 0.5 & & 57.08 $\pm$ 0.2 & 75.3 $\pm$ 0.6 \\
Nonlinear &\textbf{47.07 $\pm$ 0.1} &\textbf{68.25 $\pm$ 0.3} &&\textbf{59.04 $\pm$ 0.2} &\textbf{78.08 $\pm$ 0.2}\\
\bottomrule
\end{tabular}
\label{cub_tab:results_few}
\end{wraptable}

We evaluate \AuxiLearn{} in this setup of fine-grained classification using the Caltech-UCSD Birds 200-2011 dataset (CUB) \citep{WahCUB_200_2011}. CUB contains 200 bird species in 11,788 images, each associated with a set of 312 binary visual attributes, which we use as auxiliaries. Since we are interested in setups where optimizing the main task alone does not generalize well, we demonstrate our method in a semi-supervised setting: we assume that auxiliary labels are available for all images but only $K$ labels per class are available for the main task (noted as $K$-shot).



We compare \AuxiLearn{} with the following MTL and auxiliary learning baselines:
\noindent \textbf{(1) Single-task learning (STL):} Training only on the main task.
\noindent \textbf{(2) Equal:} Standard multitask learning with equal weights for all auxiliary tasks.
\noindent \textbf{(3) GradNorm} \citep{chen2017gradnorm}: An MTL method that scales losses based on gradient magnitude.
\noindent \textbf{(4) Uncertainty} \citep{kendall2018multi}: An MTL approach that uses task uncertainty to adjust task weights. 
\noindent \textbf{(5) Gradient Cosine Similarity (GCS)} \citep{du2018adapting}: An auxiliary-learning approach that uses gradient similarity between the main and auxiliary tasks. 
\noindent \textbf{(6) Dynamic weight averaging (DWA)}   \citep{liu2019end}: An MTL approach that sets task weights based on the rate of loss change over time. The primary network in all experiments is ResNet-18 \citep{he2016deep} pre-trained on ImageNet. 
We use a 5-layer fully connected NN for the auxiliary network. Sensitivity analysis of the network size and auxiliary set size is presented in Appendix~\ref{sec:cub_analysis}.

Table \ref{cub_tab:results_few} shows the test set classification accuracy. Most methods significantly improve over the STL baseline, highlighting the benefits of using additional (weak) labels. Our \textit{Nonlinear} and \textit{Deep linear} \GuideNet{} variants outperform all previous approaches by a large margin.
As expected, a non-linear \GuideNet{} is better than its linear counterparts. This suggests that there are some non-linear interactions between the loss terms that the non-linear network is able to capture. Also, notice the effect of using deep-linear compared to a (shallow) linear model. This result indicates that at least part of the improvement achieved by our method is attributed to the over-parameterization of the \GuideNet. In the Appendix we further analyze properties of auxiliary networks. Appendix \ref{sec:linear_weighted_poly} visualizes the full optimization path of a linear auxiliary network over a polynomial kernel on the losses, and Appendix \ref{sec:fixed_aux_params} shows that the last state of the auxiliary network is not informative enough. 

\subsection{Pixel-wise losses} \label{sec:NYU}
\begin{wraptable}[14]{r}{0.45\linewidth}
\vspace{-10pt}
\scriptsize
\centering
\caption{Test results for semantic segmentation on NYUv2, averaged over four runs ($\pm$ SEM).
}
\setlength{\tabcolsep}{3pt}
    \begin{tabular}[!t]{l  c  c}
    \toprule
    &mIoU &Pixel acc.\\
    \midrule
    STL & $18.90 \pm 0.21$ & $54.74 \pm 0.94$\\
    Equal & $19.20 \pm 0.19$ & $55.37 \pm 1.00$\\
    \hline
    Uncertainty & $19.34 \pm 0.18$ & $55.70 \pm 0.79$\\
    DWA & $19.38 \pm 0.14$ & $55.37 \pm 0.35$\\
    GradNorm & $19.52 \pm 0.21$ & $56.70 \pm 0.33$\\
    \revision{MGDA} & $19.53 \pm 0.35$ & $56.28 \pm 0.46$\\
    GCS & $19.94 \pm 0.13$ & $56.58 \pm 0.81$\\
    \midrule
    \textbf{\AuxiLearn{} (ours)} & & \\
    Linear   & $20.04 \pm 0.38$ & $\mathbf{56.80 \pm 0.14}$\\
    Deep Linear & $19.94 \pm 0.12$ & $56.45 \pm 0.79$\\
    Nonlinear& $20.09 \pm 0.34$ &
    $\mathbf{56.80 \pm 0.53}$\\
    ConvNet & $\mathbf{20.54 \pm 0.30}$ & $56.69 \pm 0.44$\\
    \bottomrule
    \end{tabular}
\label{Tab:nyu_results}
\end{wraptable}
We consider the indoor-scene segmentation task from \citet{couprie2013indoor}, that uses the NYUv2 dataset \citep{Silberman:ECCV12}. We consider the 13-class semantic segmentation as the main task, with depth and surface-normal prediction \citep{eigen2015predicting} as auxiliaries. We use SegNet~\citep{badrinarayanan2017segnet} based model
for the primary network, and a 4-layer CNN for the auxiliary network.

Since losses in this task are given per-pixel, we can apply the ConvNet variant of the \GuideNet{} to the loss image. Namely, each task forms a channel with the per-pixel losses as values. Table~\ref{Tab:nyu_results} reports the mean Intersection over Union (mIoU) and pixel accuracy for the main segmentation task. \revision{Here, we also compare with MGDA~\citep{sener2018multi} which 
had extremely long training time in CUB experiments due to the large number of auxiliary tasks, and therefore was not evaluated in Section~\ref{sec:cub_exp}}. 
All weighting methods achieve a performance gain over the STL model. The ConvNet variant of \AuxiLearn{} outperforms all competitors in terms of test mIoU.

\fig{fig:nyu_loss_grad} shows examples of the loss-images for the auxiliary (c) and main (d) tasks, together with the pixel-wise weights (e). First, note how the loss-images resemble the actual input images. This suggests that a spatial relationship can be leveraged using a CNN auxiliary network. Second, pixel weights are a non-trivial combination of the main and auxiliary task losses. In the top (bottom) row, the plant (couch) has a low segmentation loss and intermediate auxiliary loss. As a result, a higher weight is allocated to these pixels which increases the error signal.

\subsection{Learning Auxiliary labels} 
\label{sec:learning_cls_exp}

\setlength{\tabcolsep}{2pt}
\begin{table}[!h]\small
\centering
\caption{Learning auxiliary task. Test accuracy averaged over three runs ($\pm$SEM) \revision{without pre-training.}}
\scalebox{0.7}{
\begin{tabular}[!b]{l cc  cc  cc  cc  cc  cc  cc}
    \toprule
    & \multicolumn{1}{c}{CIFAR10 (5\%)} & &\multicolumn{1}{c}{CIFAR100 (5\%)} & &\multicolumn{1}{c}{SVHN (5\%)} &
    &\multicolumn{1}{c}{CUB (30-shot)} & &\multicolumn{1}{c}{Pet (30-shot)} & &\multicolumn{1}{c}{Cars (30-shot)}\\
    \midrule
    STL & $50.8 \pm 0.8$ & &  $19.8 \pm 0.7$  & & $72.9 \pm 0.3$ & & $37.2 \pm 0.8$ & & $26.1 \pm 0.5$ & &  $59.2 \pm 0.4$\\
    MAXL-F & $56.1 \pm 0.1$ & & $20.4 \pm 0.6$  & & $75.4 \pm 0.3$ & & $39.6 \pm 1.3$ & & $26.2 \pm 0.3$ & &  $59.6 \pm 1.1$\\
    MAXL & $58.2 \pm 0.3$  && $21.0 \pm 0.4$ && $75.5 \pm 0.4$ & & $40.7 \pm 0.6$ & &  $26.3 \pm 0.6$ & & $60.4 \pm 0.8$\\
    \midrule
    \textbf{\AuxiLearn{}} & $\mathbf{60.7 \pm 1.3}$
    && $\mathbf{21.5 \pm 0.3}$ && $ \mathbf{76.4 \pm 0.2}$ && $\mathbf{44.5 \pm 0.3}$ & & $\mathbf{37.0 \pm 0.6}$ & & $\mathbf{64.4 \pm 0.3}$\\
    \bottomrule
\end{tabular}}
\label{Tab:learnable_task_results_2}
\end{table}

In many cases, designing helpful auxiliaries is challenging. We now evaluate AuxiLearn in learning multi-class classification auxiliary tasks. We use three multi-class classification datasets: CIFAR10, CIFAR100 \citep{krizhevsky2009learning}, SVHN \citep{netzer2011reading}, and three fine-grained classification datasets: CUB-200-2011, Oxford-IIIT Pet \citep{parkhi12a}, and Cars \citep{KrauseStarkDengFei_Fei_3DRR2013}. Pet contains 7349 images of 37 species of dogs and cats, and Cars contains 16,185 images of 196 cars. 

Following \citet{liu2019self}, we learn a different auxiliary task for each class of the main task. 
In all experiments and all learned tasks, we set the number of classes to 5. To examine the effect of the learned auxiliary losses in the low-data regime, we evaluate the performance while training with only 5\% of the training set in CIFAR10, CIFAR100, and SVHN datasets, and $\sim 30$ samples per class in CUB, Oxford-IIIT Pet, and Cars. We use  VGG-16~\citep{simonyan2014very} as the backbone for both CIFAR datasets, a 4-layers ConvNet for the SVHN experiment, and ResNet18 for the fine-grained datasets. In all experiments, the architectures of the auxiliary and primary networks were set the same and were trained from scratch without pre-training.

We compared our approach with the following baselines:
\noindent \textbf{(1) Single-task learning (STL):} Training the main task only. \noindent \textbf{(2) MAXL:} Meta AuXiliary Learning (MAXL) proposed by~\citet{liu2019self} for learning auxiliary tasks. MAXL optimizes the label generator in a meta-learning fashion. 
\textbf{(3) MAXL-F:} A frozen MAXL label generator, that is initialized randomly. It decouples the effect of having a teacher network from the additional effect brought by the training process.

\revision{
Table~\ref{Tab:learnable_task_results_2} shows that \AuxiLearn{} outperforms all baselines in all setups, even-though it sacrifices some of the training set for the \metaVal{}. It is also worth noting that our optimization approach is significantly faster than MAXL, yielding $\times3$ improvement in run-time.} 
In Appendix ~\ref{sec:pc} and \ref{sec:learning_aux_furter} we show additional experiments for this setup, including an extension of the method to point-cloud part segmentation and experiments with varying training data sizes.

\iclrrevision{
\figref{fig:cifar10_labels} presents a 2D t-SNE projection of the 5D vector of auxiliary (soft) labels that are learned using AuxiLearn. We use samples of the main classes \textit{Frog} (left) and \textit{Deer} (right) from the CIFAR10 dataset. t-SNE was applied to each auxiliary task separately. When considering how images are projected to this space of auxiliary soft labels, several structures emerge. The \GuideNet{} learns a fine partition of the \textit{Frog} class that separates real images from illustrations. More interesting, the soft labels learned for the class \textit{Deer} have a middle region that only contains deers with antlers (in various poses and varying backgrounds). By capturing this semantic feature in the learned auxiliary labels, the auxiliary task can help the primary network to discriminate between main task classes.}

\begin{figure}[t]
    \text{
        }\par
    \centering
    \includegraphics[width=0.68\linewidth]{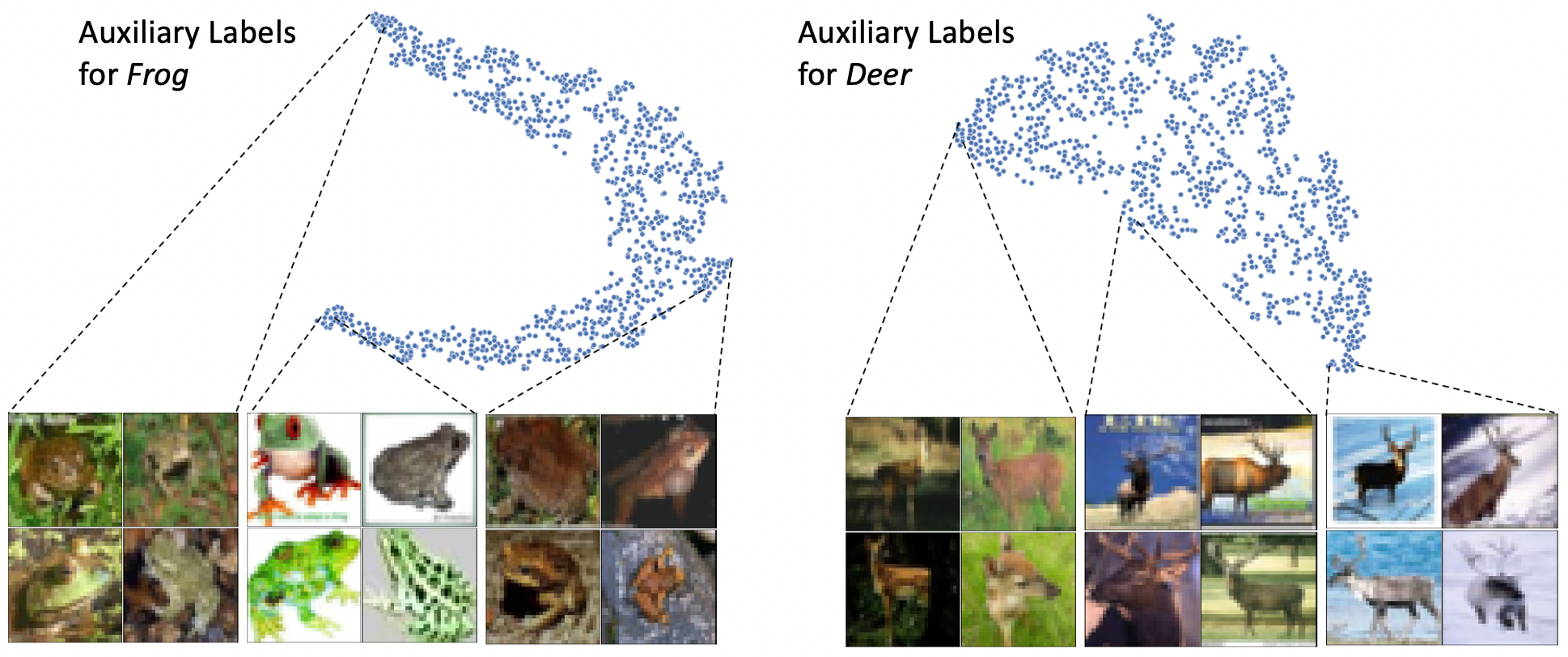}
    \caption{t-SNE applied to auxiliary labels learned for \textit{Frog} and \textit{Deer} classes, in CIFAR10. Best viewed in color.}
    \label{fig:cifar10_labels}
\end{figure}

\section{Discussion}
\revision{In this paper, we presented a novel and unified approach for two tasks: combining predefined auxiliary tasks, and learning auxiliary tasks that are useful for the primary task. We theoretically showed which auxiliaries can be beneficial and the importance of using a separate auxiliary set. We empirically demonstrated that our method achieves significant improvement over existing methods on various datasets and tasks. This work opens interesting directions for future research. First, when training deep linear auxiliary networks, we observed similar learning dynamics to those of non-linear models. As a result, they generated better performance compared to their linear counterparts. This effect was observed in standard training setup, but the optimization path in auxiliary networks is very different. Second, we find that reallocating labeled data from the training set to an \metaVal{} is consistently helpful. A broader question remains what is the most efficient allocation.}

\section*{Acknowledgements}
This study was funded by a grant to GC from the Israel Science Foundation (ISF 737/2018), and by an equipment grant to GC and Bar-Ilan University from the Israel Science Foundation (ISF 2332/18). IA and AN were funded by a grant from the Israeli innovation authority, through the AVATAR consortium.

\newpage


\bibliography{ref}

\begin{thebibliography}{}

\bibitem[Achituve et~al., 2020]{achituve2020self}
Achituve, I., Maron, H., and Chechik, G. (2020).
\newblock Self-supervised learning for domain adaptation on point-clouds.
\newblock {\em arXiv preprint arXiv:2003.12641}.

\bibitem[Alemi et~al., 2017]{IB}
Alemi, A.~A., Fischer, I., Dillon, J.~V., and Murphy, K. (2017).
\newblock Deep variational information bottleneck.
\newblock In {\em International Conference on Learning Representations}.

\bibitem[Arora et~al., 2018]{arora2018optimization}
Arora, S., Cohen, N., and Hazan, E. (2018).
\newblock On the optimization of deep networks: Implicit acceleration by
  overparameterization.
\newblock In {\em International Conference on Machine Learning}.

\bibitem[Badrinarayanan et~al., 2017]{badrinarayanan2017segnet}
Badrinarayanan, V., Kendall, A., and Cipolla, R. (2017).
\newblock Segnet: A deep convolutional encoder-decoder architecture for image
  segmentation.
\newblock {\em IEEE transactions on pattern analysis and machine intelligence},
  39(12):2481--2495.

\bibitem[Bengio, 2000]{bengio2000gradient}
Bengio, Y. (2000).
\newblock Gradient-based optimization of hyperparameters.
\newblock {\em Neural computation}, 12(8):1889--1900.

\bibitem[Chang et~al., 2015]{chang2015shapenet}
Chang, A.~X., Funkhouser, T., Guibas, L., Hanrahan, P., Huang, Q., Li, Z.,
  Savarese, S., Savva, M., Song, S., Su, H., et~al. (2015).
\newblock Shapenet: An information-rich 3d model repository.
\newblock {\em arXiv preprint arXiv:1512.03012}.

\bibitem[Chen et~al., 2018]{chen2017gradnorm}
Chen, Z., Badrinarayanan, V., Lee, C.-Y., and Rabinovich, A. (2018).
\newblock Gradnorm: Gradient normalization for adaptive loss balancing in deep
  multitask networks.
\newblock In {\em International Conference on Machine Learning}, pages
  794--803. PMLR.

\bibitem[Cordts et~al., 2016]{cordts2016cityscapes}
Cordts, M., Omran, M., Ramos, S., Rehfeld, T., Enzweiler, M., Benenson, R.,
  Franke, U., Roth, S., and Schiele, B. (2016).
\newblock The cityscapes dataset for semantic urban scene understanding.
\newblock In {\em Proceedings of the IEEE conference on computer vision and
  pattern recognition}, pages 3213--3223.

\bibitem[Couprie et~al., 2013]{couprie2013indoor}
Couprie, C., Farabet, C., Najman, L., and LeCun, Y. (2013).
\newblock Indoor semantic segmentation using depth information.
\newblock In {\em International Conference on Learning Representations}.

\bibitem[Deng et~al., 2009]{deng2009imagenet}
Deng, J., Dong, W., Socher, R., Li, L.-J., Li, K., and Fei-Fei, L. (2009).
\newblock Imagenet: A large-scale hierarchical image database.
\newblock In {\em Proceedings of the IEEE conference on computer vision and
  pattern recognition}, pages 248--255.

\bibitem[Doersch et~al., 2015]{doersch2015unsupervised}
Doersch, C., Gupta, A., and Efros, A.~A. (2015).
\newblock Unsupervised visual representation learning by context prediction.
\newblock In {\em Proceedings of the IEEE International Conference on Computer
  Vision}, pages 1422--1430.

\bibitem[Doersch and Zisserman, 2017]{doersch2017multi}
Doersch, C. and Zisserman, A. (2017).
\newblock Multi-task self-supervised visual learning.
\newblock In {\em Proceedings of the IEEE International Conference on Computer
  Vision}, pages 2051--2060.

\bibitem[Du et~al., 2018]{du2018adapting}
Du, Y., Czarnecki, W.~M., Jayakumar, S.~M., Pascanu, R., and Lakshminarayanan,
  B. (2018).
\newblock Adapting auxiliary losses using gradient similarity.
\newblock {\em arXiv preprint arXiv:1812.02224}.

\bibitem[Eigen and Fergus, 2015]{eigen2015predicting}
Eigen, D. and Fergus, R. (2015).
\newblock Predicting depth, surface normals and semantic labels with a common
  multi-scale convolutional architecture.
\newblock In {\em Proceedings of the IEEE international conference on computer
  vision}, pages 2650--2658.

\bibitem[Fan et~al., 2017]{fan2017transfer}
Fan, X., Monti, E., Mathias, L., and Dreyer, M. (2017).
\newblock Transfer learning for neural semantic parsing.
\newblock In {\em Proceedings of the 2nd Workshop on Representation Learning
  for NLP}, pages 48--56.

\bibitem[Finn et~al., 2017]{finn2017model}
Finn, C., Abbeel, P., and Levine, S. (2017).
\newblock Model-agnostic meta-learning for fast adaptation of deep networks.
\newblock In {\em International Conference on Machine Learning}, pages
  1126--1135.

\bibitem[Foo et~al., 2008]{foo2008efficient}
Foo, C.-s., Do, C.~B., and Ng, A.~Y. (2008).
\newblock Efficient multiple hyperparameter learning for log-linear models.
\newblock In {\em Advances in neural information processing systems}, pages
  377--384.

\bibitem[Ganin and Lempitsky, 2015]{DA}
Ganin, Y. and Lempitsky, V. (2015).
\newblock Unsupervised domain adaptation by backpropagation.
\newblock In {\em International Conference on Machine Learning}.

\bibitem[Gidaris et~al., 2018]{gidaris2018unsupervised}
Gidaris, S., Singh, P., and Komodakis, N. (2018).
\newblock Unsupervised representation learning by predicting image rotations.
\newblock In {\em International Conference on Learning Representations}.

\bibitem[Goyal et~al., 2019]{Goyal_2019_ICCV}
Goyal, P., Mahajan, D., Gupta, A., and Misra, I. (2019).
\newblock Scaling and benchmarking self-supervised visual representation
  learning.
\newblock In {\em Proceedings of the IEEE/CVF International Conference on
  Computer Vision (ICCV)}.

\bibitem[Hassani and Haley, 2019]{hassani2019unsupervised}
Hassani, K. and Haley, M. (2019).
\newblock Unsupervised multi-task feature learning on point clouds.
\newblock In {\em Proceedings of the IEEE International Conference on Computer
  Vision}, pages 8160--8171.

\bibitem[He et~al., 2016]{he2016deep}
He, K., Zhang, X., Ren, S., and Sun, J. (2016).
\newblock Deep residual learning for image recognition.
\newblock In {\em Proceedings of the IEEE conference on computer vision and
  pattern recognition}, pages 770--778.

\bibitem[Jaderberg et~al., 2016]{jaderberg2016reinforcement}
Jaderberg, M., Mnih, V., Czarnecki, W.~M., Schaul, T., Leibo, J.~Z., Silver,
  D., and Kavukcuoglu, K. (2016).
\newblock Reinforcement learning with unsupervised auxiliary tasks.
\newblock {\em arXiv preprint arXiv:1611.05397}.

\bibitem[Jing and Tian, 2020]{jing2020self}
Jing, L. and Tian, Y. (2020).
\newblock Self-supervised visual feature learning with deep neural networks: A
  survey.
\newblock {\em IEEE Transactions on Pattern Analysis and Machine Intelligence}.

\bibitem[Kendall et~al., 2018]{kendall2018multi}
Kendall, A., Gal, Y., and Cipolla, R. (2018).
\newblock Multi-task learning using uncertainty to weigh losses for scene
  geometry and semantics.
\newblock In {\em Proceedings of the IEEE conference on computer vision and
  pattern recognition}, pages 7482--7491.

\bibitem[Kingma and Ba, 2014]{Kingma2014AdamAM}
Kingma, D.~P. and Ba, J. (2014).
\newblock {ADAM}: A method for stochastic optimization.
\newblock In {\em International Conference on Learning Representations}.

\bibitem[Krause et~al., 2013]{KrauseStarkDengFei_Fei_3DRR2013}
Krause, J., Stark, M., Deng, J., and Fei-Fei, L. (2013).
\newblock {3D} object representations for fine-grained categorization.
\newblock In {\em 4th International IEEE Workshop on 3D Representation and
  Recognition}, Sydney, Australia.

\bibitem[Krizhevsky et~al., 2009]{krizhevsky2009learning}
Krizhevsky, A., Hinton, G., et~al. (2009).
\newblock Learning multiple layers of features from tiny images.
\newblock Technical report, University of Toronto.

\bibitem[Larsen et~al., 1996]{larsen1996design}
Larsen, J., Hansen, L.~K., Svarer, C., and Ohlsson, M. (1996).
\newblock Design and regularization of neural networks: the optimal use of a
  validation set.
\newblock In {\em Neural Networks for Signal Processing VI. Proceedings of the
  IEEE Signal Processing Society Workshop}, pages 62--71. IEEE.

\bibitem[Liao et~al., 2018]{liao18}
Liao, R., Xiong, Y., Fetaya, E., Zhang, L., Yoon, K., Pitkow, X., Urtasun, R.,
  and Zemel, R. (2018).
\newblock Reviving and improving recurrent back-propagation.
\newblock In {\em International Conference on Machine Learning}.

\bibitem[Lin et~al., 2019]{lin2019adaptive}
Lin, X., Baweja, H., Kantor, G., and Held, D. (2019).
\newblock Adaptive auxiliary task weighting for reinforcement learning.
\newblock In {\em Advances in Neural Information Processing Systems}, pages
  4773--4784.

\bibitem[Liu et~al., 2019a]{liu2019self}
Liu, S., Davison, A., and Johns, E. (2019a).
\newblock Self-supervised generalisation with meta auxiliary learning.
\newblock In {\em Advances in Neural Information Processing Systems}, pages
  1677--1687.

\bibitem[Liu et~al., 2019b]{liu2019end}
Liu, S., Johns, E., and Davison, A.~J. (2019b).
\newblock End-to-end multi-task learning with attention.
\newblock In {\em Proceedings of the IEEE Conference on Computer Vision and
  Pattern Recognition}, pages 1871--1880.

\bibitem[Lorraine et~al., 2020]{lorraine2019optimizing}
Lorraine, J., Vicol, P., and Duvenaud, D. (2020).
\newblock Optimizing millions of hyperparameters by implicit differentiation.
\newblock In {\em International Conference on Artificial Intelligence and
  Statistics}, pages 1540--1552. PMLR.

\bibitem[Luketina et~al., 2016]{luketina2016scalable}
Luketina, J., Berglund, M., Greff, K., and Raiko, T. (2016).
\newblock Scalable gradient-based tuning of continuous regularization
  hyperparameters.
\newblock In {\em International conference on machine learning}, pages
  2952--2960.

\bibitem[Mirowski, 2019]{mirowski2019learning}
Mirowski, P. (2019).
\newblock Learning to navigate.
\newblock In {\em 1st International Workshop on Multimodal Understanding and
  Learning for Embodied Applications}, pages 25--25.

\bibitem[Netzer et~al., 2011]{netzer2011reading}
Netzer, Y., Wang, T., Coates, A., Bissacco, A., Wu, B., and Ng, A.~Y. (2011).
\newblock Reading digits in natural images with unsupervised feature learning.

\bibitem[Noroozi and Favaro, 2016]{noroozi2016unsupervised}
Noroozi, M. and Favaro, P. (2016).
\newblock Unsupervised learning of visual representations by solving jigsaw
  puzzles.
\newblock In {\em Proceedings of the European Conference on Computer Vision},
  pages 69--84. Springer.

\bibitem[Parkhi et~al., 2012]{parkhi12a}
Parkhi, O.~M., Vedaldi, A., Zisserman, A., and Jawahar, C.~V. (2012).
\newblock Cats and dogs.
\newblock In {\em IEEE Conference on Computer Vision and Pattern Recognition}.

\bibitem[Pedregosa, 2016]{pedregosa2016hyperparameter}
Pedregosa, F. (2016).
\newblock Hyperparameter optimization with approximate gradient.
\newblock In {\em International Conference on Machine Learning}, pages
  737--746.

\bibitem[Qi et~al., 2017]{qi2017pointnet}
Qi, C.~R., Su, H., Mo, K., and Guibas, L.~J. (2017).
\newblock Pointnet: Deep learning on point sets for 3d classification and
  segmentation.
\newblock In {\em Proceedings of the IEEE conference on computer vision and
  pattern recognition}, pages 652--660.

\bibitem[Rajeswaran et~al., 2019]{rajeswaran2019meta}
Rajeswaran, A., Finn, C., Kakade, S.~M., and Levine, S. (2019).
\newblock Meta-learning with implicit gradients.
\newblock In {\em Advances in Neural Information Processing Systems}, pages
  113--124.

\bibitem[Salimans and Kingma, 2016]{salimans2016weight}
Salimans, T. and Kingma, D.~P. (2016).
\newblock Weight normalization: A simple reparameterization to accelerate
  training of deep neural networks.
\newblock In {\em Advances in neural information processing systems}, pages
  901--909.

\bibitem[Sauder and Sievers, 2019]{sauder2019self}
Sauder, J. and Sievers, B. (2019).
\newblock Self-supervised deep learning on point clouds by reconstructing
  space.
\newblock In {\em Advances in Neural Information Processing Systems}, pages
  12942--12952.

\bibitem[Saxe et~al., 2014]{saxe2014exact}
Saxe, A.~M., Mcclelland, J.~L., and Ganguli, S. (2014).
\newblock Exact solutions to the nonlinear dynamics of learning in deep linear
  neural network.
\newblock In {\em In International Conference on Learning Representations}.
  Citeseer.

\bibitem[Sener and Koltun, 2018]{sener2018multi}
Sener, O. and Koltun, V. (2018).
\newblock Multi-task learning as multi-objective optimization.
\newblock In {\em Advances in Neural Information Processing Systems}, pages
  527--538.

\bibitem[Silberman et~al., 2012]{Silberman:ECCV12}
Silberman, N., Hoiem, D., Kohli, P., and Fergus, R. (2012).
\newblock Indoor segmentation and support inference from {RGBD} images.
\newblock In {\em Proceedings of the European conference on computer vision},
  pages 746--760. Springer.

\bibitem[Simonyan and Zisserman, 2014]{simonyan2014very}
Simonyan, K. and Zisserman, A. (2014).
\newblock Very deep convolutional networks for large-scale image recognition.
\newblock {\em arXiv preprint arXiv:1409.1556}.

\bibitem[Standley et~al., 2019]{standley2019tasks}
Standley, T., Zamir, A.~R., Chen, D., Guibas, L., Malik, J., and Savarese, S.
  (2019).
\newblock Which tasks should be learned together in multi-task learning?
\newblock {\em arXiv preprint arXiv:1905.07553}.

\bibitem[Tang et~al., 2020]{tang2020improving}
Tang, L., Chen, K., Wu, C., Hong, Y., Jia, K., and Yang, Z. (2020).
\newblock Improving semantic analysis on point clouds via auxiliary supervision
  of local geometric priors.
\newblock {\em arXiv preprint arXiv:2001.04803}.

\bibitem[Trinh et~al., 2018]{trinh2018learning}
Trinh, T., Dai, A., Luong, T., and Le, Q. (2018).
\newblock Learning longer-term dependencies in {RNN}s with auxiliary losses.
\newblock In {\em International Conference on Machine Learning}, pages
  4965--4974.

\bibitem[Vinyals et~al., 2016]{vinyals2016matching}
Vinyals, O., Blundell, C., Lillicrap, T., Wierstra, D., et~al. (2016).
\newblock Matching networks for one shot learning.
\newblock In {\em Advances in neural information processing systems}, pages
  3630--3638.

\bibitem[Wah et~al., 2011]{WahCUB_200_2011}
Wah, C., Branson, S., Welinder, P., Perona, P., and Belongie, S. (2011).
\newblock {The Caltech-UCSD Birds-200-2011 Dataset}.
\newblock Technical Report CNS-TR-2011-001, California Institute of Technology.

\bibitem[Wang et~al., 2019]{wang2019dynamic}
Wang, Y., Sun, Y., Liu, Z., Sarma, S.~E., Bronstein, M.~M., and Solomon, J.~M.
  (2019).
\newblock Dynamic graph cnn for learning on point clouds.
\newblock {\em ACM Transactions on Graphics}, 38(5):1--12.

\bibitem[Yi et~al., 2016]{yi2016scalable}
Yi, L., Kim, V.~G., Ceylan, D., Shen, I.-C., Yan, M., Su, H., Lu, C., Huang,
  Q., Sheffer, A., and Guibas, L. (2016).
\newblock A scalable active framework for region annotation in {3D} shape
  collections.
\newblock {\em ACM Transactions on Graphics}, 35(6):1--12.

\bibitem[Zhang and Yang, 2017]{zhang2017survey}
Zhang, Y. and Yang, Q. (2017).
\newblock A survey on multi-task learning.
\newblock {\em arXiv preprint arXiv:1707.08114}.

\bibitem[Zhang et~al., 2014]{zhang2014facial}
Zhang, Z., Luo, P., Loy, C.~C., and Tang, X. (2014).
\newblock Facial landmark detection by deep multi-task learning.
\newblock In {\em European conference on computer vision}, pages 94--108.
  Springer.

\end{thebibliography}
\bibliographystyle{apalike}


\newpage

{\LARGE{Appendix: Auxiliary Learning by Implicit Differentiation}}

\appendix
\section{Gradient derivation}
\label{app:ift_derivation}
We provide here the derivation of Eq.~\eqref{eq:ift_dw_dphi} in \secref{sec:approach}. One can look at the function $\nabla_W\LL_T(W,\phi)$ around a certain local-minima point $(\hat{W},\hat{\phi})$ and assume the Hessian $\nabla_W^2\LL_T(\hat{W},\hat{\phi})$ is positive-definite. At that point, we have $\nabla_W\LL_T(\hat{W},\hat{\phi})=0$. From the IFT,
we have that locally around $(\hat{W},\hat{\phi})$,  there exists a smooth function $W^*(\phi)$ such that  $\nabla_W\LL_T({W},{\phi})=0$ if $W=W^*(\phi)$. Since the function $\nabla_W\LL_T(W^*(\phi),\phi)$ is constant and equal to zero, we have that its derivative w.r.t. $\phi$ is also zero. Taking the total derivative we obtain
\begin{equation}
    0=\nabla_W^2\LL_T({W},{\phi})\nabla_\phi W^*(\phi)+\nabla_\phi\nabla_W\LL_T(W,\phi)\,.
\end{equation}
Multiplying by $\nabla_W^2\LL_T({W},{\phi})^{-1}$ and reordering we obtain
\begin{equation}
    \nabla_\phi W^*(\phi)=-\nabla_W^2\LL_T({W},{\phi})^{-1}\nabla_\phi\nabla_W\LL_T(W,\phi)\,.
\end{equation}
We can use this result to compute the gradients of the \metaVal{} loss w.r.t $\phi$
\begin{equation}
    \label{eq:ift_full}
    \nabla_\phi{\LL_A(W^*(\phi))}=\nabla_W\LL_A \cdot\nabla_\phi W^*(\phi)=-\nabla_W\LL_A \cdot (\nabla^2_W\LL_T)^{-1} \cdot\nabla_\phi\nabla_W\LL_T .
\end{equation}

As discussed in the main text, fully optimizing $W$ to convergence is too computationally expensive. Instead, we update $\phi$  once for every several update steps for $W$, as seen in Alg.~\ref{alg:auxilearn}. To compute the vector inverse-Hessian product, we use Alg.~\ref{alg:hypergrad} that was proposed in~\citep{lorraine2019optimizing}.



\section{Experimental details}
 \label{training_details}

\subsection{CUB 200-2011} \label{sec:cub_dataset_desc}
\textbf{Data.} 
\revision{To examine the effect of varying training set sizes we use all 5994  predefined images for training according to the official split and, we split the predefined test set to 2897 samples for validation and 2897 for testing.}
All images were resized to $256\times256$ and Z-score normalized. During training, images were randomly cropped to $224$ and flipped horizontally. Test images were centered cropped to $224$. The same processing was applied in all fine-grain experiments. 

\textbf{Training details for baselines.} We fine-tuned a ResNet-18 \citep{he2016deep} pre-trained on ImageNet \citep{deng2009imagenet} with a classification layer on top for all tasks. Because the scale of auxiliary losses differed from that of the main task, we multiplied each auxiliary loss, on all compared method, by the scaling factor $\tau = 0.1$. It was chosen based on a grid search over $\{0.1, 0.3, 0.6, 1.0\}$ using the \textit{Equal} baseline. We applied grid search over the learning rates in $\{1e-3, 1e-4, 1e-5\}$ and the weight decay in $\{5e-3, 5e-4, 5e-5\}$. For DWA \citep{liu2019end}, we searched over the temperature in $\{0.5, 2, 5\}$ and for GradNorm \citep{chen2017gradnorm}, over $\alpha$ in $\{0.3, 0.8, 1.5\}$. The computational complexity of GSC~\citep{du2018adapting} grows with the number of tasks. As a result, we were able to run this baseline only in a setup where there are two loss terms: the main and the sum of all auxiliary tasks. We ran each configuration with $3$ different seeds for $100$ epochs with ADAM optimizer \citep{Kingma2014AdamAM} and used early stopping based on the validation set.

\textbf{The \metaVal{} and auxiliary network.} In our experiments, we found that allocating as little as 20 samples from the training set for the \metaVal{} and using a NN with 5 layers and 10 units in each layer yielded good performance for both deep linear and non-linear models. We found that our method was not sensitive to these design choices. We use skip connection between the main loss $\ell_{main}$ and the overall loss term and Softplus activation.

\textbf{Optimization of the \GuideNet{}.}
In all variants of our method, the \GuideNet{} was optimized using SGD with 0.9 momentum. We applied grid search over the \GuideNet{} learning rate in $\{1e-2, 1e-3\}$ and weight decay in $\{1e-5, 5e-5\}$. The total training time of all methods was 3 hours on a 16GB Nvidia V100 GPU.



\subsection{NYUv2}\label{sec:nyu_details}
The data consists of 1449 RGB-D images, split into 795 train images and 654 test images. We further split the train set to allocate 79 images, 10\% of training examples, to construct a validation set. Following \citep{liu2019end}, we resize images to $288 \times 384$ pixels for training and evaluation and use SegNet \citep{badrinarayanan2017segnet} based architecture as the backbone. 

Similar to~\citep{liu2019end}, we train the model for $200$ epochs using Adam optimizer~\citep{Kingma2014AdamAM} with learning rate $1e-4$, and halve the learning rate after $100$ epochs. We choose the best model with early stopping on a pre-allocated validation set. For DWA \citep{liu2019end} we set the temperature hyperparameter to 2, as in the NYUv2 experiment in~\citep{liu2019end}. For GradNorm~\citep{chen2017gradnorm} we set 
$\alpha=1.5$. This value for $\alpha$ was used in~\citep{chen2017gradnorm} for the NYUv2 experiments. In all variants of our method, the \GuideNet{}s are optimized using SGD with 0.9 momentum. We allocate $2.5\%$ of training examples to form an \metaVal{}. We use grid search to tune the learning rate $\{1e-3, 5e-4, 1e-4\}$ and weight decay $\{1e-5, 1e-4\}$ of the \GuideNet{}s. \revision{Here as well, we use skip connection between the main loss $\ell_{main}$ and the overall loss term and Softplus activation.}

\subsection{Learning auxiliaries}
\textbf{Multi-class classification datasets.}
On the CIFAR datasets, we train the model for $200$ epochs using SGD with momentum $0.9$, weight decay $5e-4$, and initial learning rates $1e-1$ and $1e-2$ for CIFAR10 and CIFAR100, respectively. For the SVHN experiment, we train for $50$ epochs using SGD with momentum $0.9$, weight decay $5e-4$, and initial learning rates $1e-1$. The learning rate is modified using a cosine annealing scheduler. We use VGG-16~\citep{simonyan2014very} based architecture for the CIFAR experiments, and a 4-layer ConvNet for the SVHN experiment. For MAXL~\citep{liu2019self} label generating network, we tune the following hyperparameters: learning rate $\{1e-3, 5e-4\}$, weight decay $\{5e-4, 1e-4, 5e-5\}$, and entropy term weight $\{.2, .4, .6\}$ (see~\citep{liu2019self} for details). We explore the same learning rate and weight decay for the \GuideNet{} in our method, and also tune the number of optimization steps between every auxiliary parameter update $\{5, 15, 25\}$, and the size of the \metaVal{} $\{1.5\%, 2.5\%\}$ (of training examples). We choose the best model on the validation set and allow for early stopping.

\revision{\textbf{Fine-grain classification datasets.} In CUB experiments we use the same data and splits as described in Sections~\ref{sec:cub_exp} and \ref{sec:cub_dataset_desc}. Oxford-IIIT Pet contains $7349$ images of $37$ species of dogs and cats. We use the official train-test split. We pre-allocate $30\%$ from the training set to validation. As a results, the total number of train/validation/test images are $2576/1104/3669$ respectively. Cars \citep{KrauseStarkDengFei_Fei_3DRR2013} contains $16,185$ images of $196$ car classes. We use the official train-test split and pre-allocate $30\%$ from the training set to validation. As a results, the total number of train/validation/test images are $5700/2444/8041$ respectively. In all experiments we use ResNet-18 as the backbone network for both the primary and auxiliary networks. Importantly, the networks are not pre-trained. The task specific (classification) heads in both the primary and auxiliary networks is implemented using a 2-layer NN with sizes 512 and $C$. Where $C$ is number of labels (e.g., $200$ for CUB and $37$ for Oxford-IIIT Pet). In all experiments we use the same learning rate of $1e-4$ and weight decay of $5e-3$ which were shown to work best, based on a grid search applied on the STL baseline. For MAXL and AuxiLearn we applied a grid search over the auxiliary network learning rate and weight decay as described in the Multi-class classification datasets subsection. We tune the number of optimization steps between every auxiliary parameter update in $\{30, 60\}$ for Oxford-IIIT Pet and $\{40, 80\}$ for CUB and Cars. Also, the auxiliary set size was tuned over $\{0.084\%, 1.68\%, 3.33\%\}$ with stratified sampling.
For our method, we leverage the module of AuxiLearn for combining auxiliaries. We use a Nonlinear network with either two or three hidden layers of sizes 10 (which was selected according to a grid search). The batch size was set to 64 in CUB and Cars experiments and to 16 in Oxford-IIIT Pet experiments. We ran each configuration with 3 different seeds for 150 epochs with ADAM optimizer and used early stopping based on the validation set.
}

\section{Additional experiments} \label{sec:additional_exp}

\subsection{\revision{Importance of Auxiliary Set}}\label{sup:opt_on_train}

\begin{figure}[h]
    \text{
        }\par
    \centering
    \includegraphics[width=.75\linewidth]{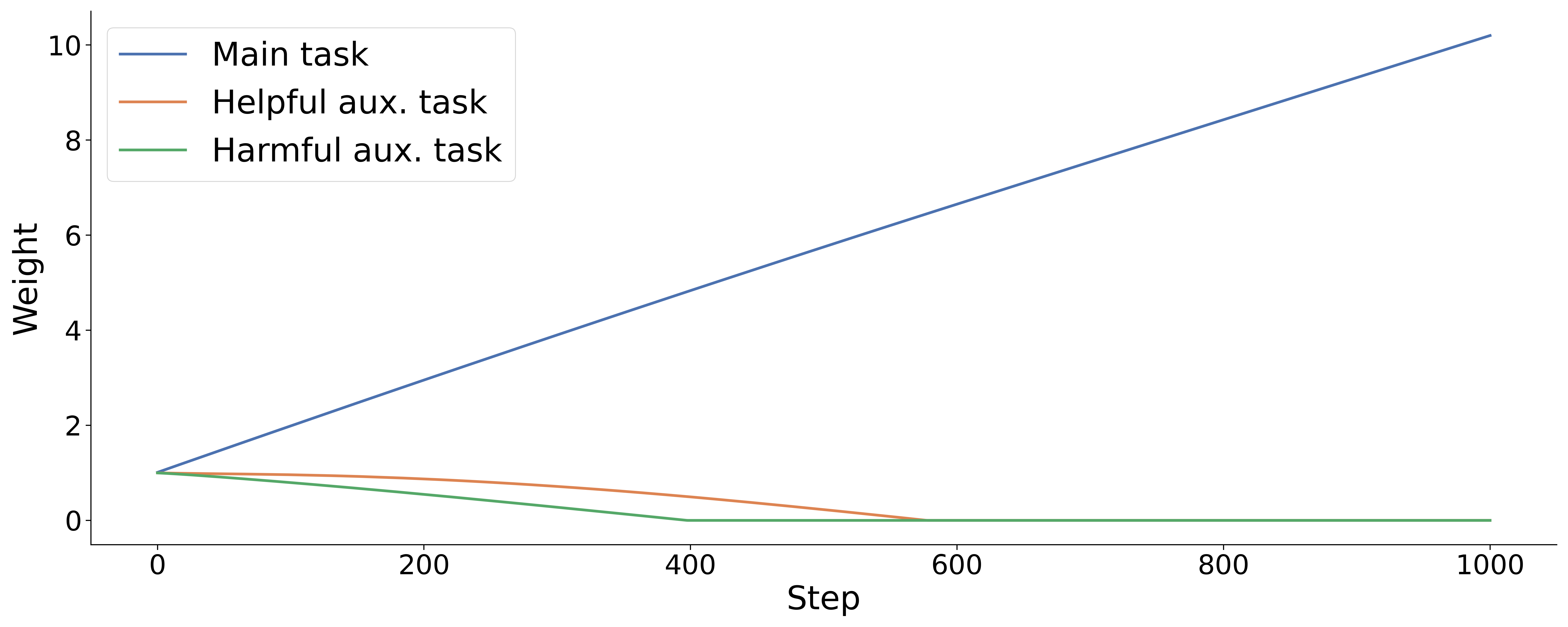}
    \caption{Optimizing task weights on the training set reduce to single-task learning.}
    \label{fig:opt_on_train}
\end{figure}

In this section we illustrate the importance of the auxiliary set to complement our theoretical observation in \secref{sec:analysis}. We repeat the experiment in \secref{sec:illustrative}, but this time we optimize the auxiliary parameters $\phi$ using the training data.
\figref{fig:opt_on_train} shows how the tasks' weights change during training. The optimization procedure is reduced to single-task learning, which badly hurts generalization (see \figref{fig:loss_landscape}). These results are consistent with~\citep{liu2019self} that added an entropy loss term to avoid the diminishing auxiliary task.




\subsection{Monotonocity}
\label{app:monotonicity}
As discussed in the main text, it is a common practice to combine auxiliary losses as a convex combination. This is equivalent to parametrize the function $g(\boldsymbol{\ell};\phi)$ as a linear combination over losses $g(\boldsymbol{\ell};\phi) = \sum_{j=1}^K \phi_j \ell_j$, with non-negative weights, $\phi_j\ge 0$. Under this  parameterization, $g$ is a monotonic non-decreasing function of the losses, since $\partial\LL_T/\partial \ell_j \ge 0$. The non-decreasing property means that the overall loss grows (or is left unchanged) with any increase to the auxiliary losses. As a result, an optimization procedure that operates to minimize the combined loss also operates in the direction of reducing individual losses (or not changing them). 

A natural question that arises is whether the function $g$  should generalize this behavior, and be constrained to be non-decreasing w.r.t. the losses as well?
Non-decreasing networks can "ignore" an auxiliary task by zeroing its corresponding loss, but cannot reverse the gradient of a task by negating its weight. While monotonicity is a very natural requirement, 
in some cases, negative task weights (i.e., non-monotonicity) seem desirable if one wishes to "delete" input information not directly related to the task at hand \citep{IB, DA}. For example, in domain adaptation, one might want to remove information that allows a discriminator to recognize the domain of a given sample  \citep{DA}. Empirically, we found that training with monotonic non-decreasing networks to be more stable and has better or equivalent performance, see Table \ref{cub_tab:mon_vs_nonmon} for comparison.

Table \ref{cub_tab:mon_vs_nonmon} compares  monotonic and non-monotonic \GuideNet{}s in both the semi-supervised and the fully-supervised setting. Monotonic networks show a small but consistent improvement over non-monotonic ones. It is also worth mentioning that the non-monotonic networks were harder to stabilize.
\setlength{\tabcolsep}{6pt}
\begin{table}[h]
\centering
\caption{CUB 200-2011: Monotonic vs non-monotonic test classification accuracy ($\pm$ SEM) over three runs. \newline}
\begin{tabular}[h]{l l cc}
\toprule
& &Top 1 &Top 3 \\\midrule
\multirow{2}{*}{5-shot} &Non-Monotonic & 46.3 $\pm$ 0.32 & 67.46 $\pm$ 0.55 \\
& Monotonic & \textbf{47.07 $\pm$ 0.10} & \textbf{68.25 $\pm$ 0.32} \\
\midrule
\multirow{2}{*}{10-shot} & Non-Monotonic & 58.84 $\pm$ 0.04 & 77.67 $\pm$ 0.08 \\
&Monotonic & \textbf{59.04 $\pm$ 0.22} & \textbf{78.08 $\pm$ 0.24} \\
\midrule
\multirow{2}{*}{Full Dataset} & Non-Monotonic & 74.74 $\pm$ 0.30 & 88.3 $\pm$ 0.23 \\
&Monotonic & \textbf{74.92 $\pm$ 0.21} & \textbf{88.55 $\pm$ 0.17} \\
\bottomrule
\end{tabular}
\label{cub_tab:mon_vs_nonmon}
\end{table}


\subsection{Noisy auxiliaries}  \label{sec:noisy_aux}
We demonstrate the effectiveness of \AuxiLearn{} in identifying helpful auxiliaries and ignoring harmful ones. Consider a regression problem with main task $y=\ww^T \xx+\epsilon$, where $\epsilon\sim \mathcal{N}(0, \sigma^2)$. We learn this task jointly with $K=100$ auxiliaries of the form $y_j=\ww^T \xx+|\epsilon_j|$, where $\epsilon_j\sim \mathcal{N}(0, j\cdot\sigma_{aux}^2)$ for $j=1, ..., 100$. We use the absolute value on the noise so that noisy estimations are no longer unbiased, making the noisy labels even less helpful as the noise increases.
We use a linear \GuideNet{} to weigh the loss terms. Figure~\ref{fig:noisy_aux} shows the learned weight for each task. We can see that the \GuideNet{} captures the noise patterns, and assign weights based on the noise level.

\begin{figure}[h]
    \text{
        }\par
    \centering
    \includegraphics[width=1.\linewidth]{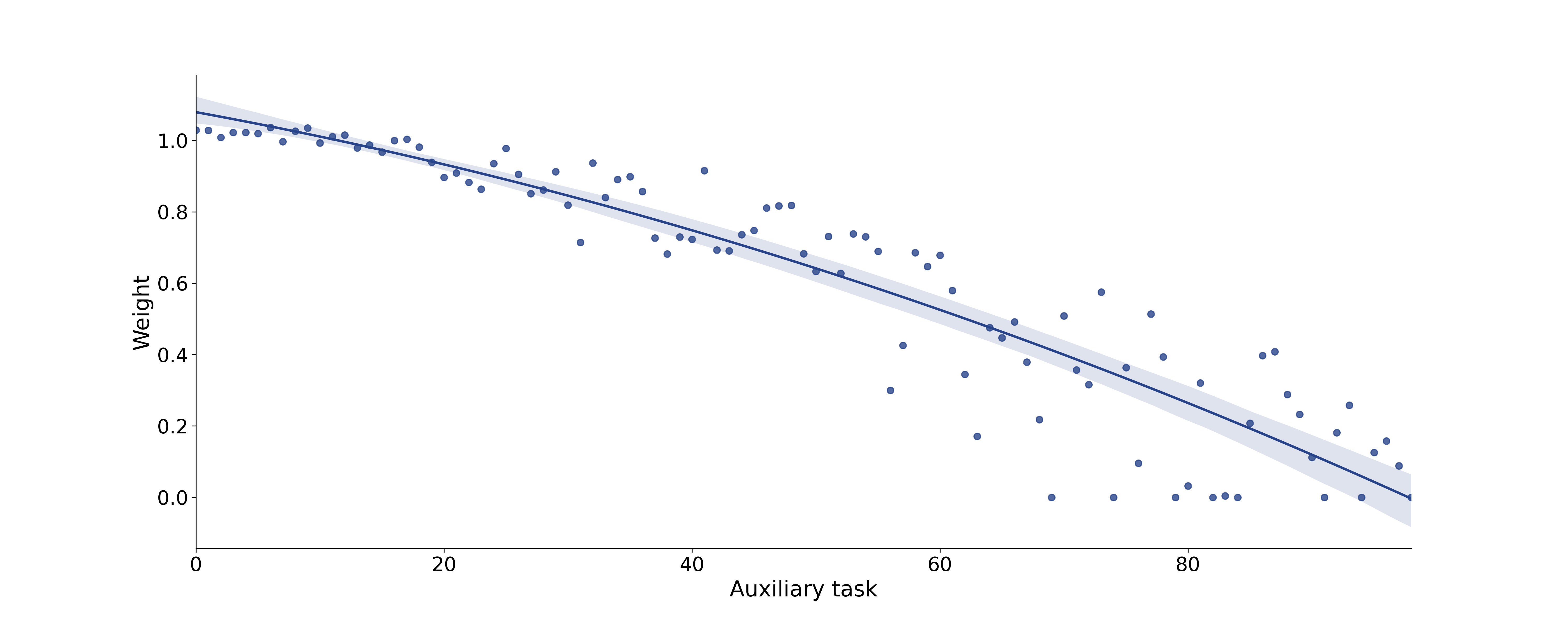}
    \caption{Learning with noisy labels: task ID is proportional to the label noise.}
    \label{fig:noisy_aux}
\end{figure}


\subsection{CUB sensitivity analysis} 
\label{sec:cub_analysis}
In this section, we provide further analysis for the experiments conducted on the CUB 200-2011 dataset in the 5-shot setup. We examine the sensitivity of a non-linear \GuideNet{} to the \textbf{size of the \metaVal{}}, and the \textbf{depth of the \GuideNet{}}. In~\figref{fig:cub_meta-val_effect} we test the effect of allocating (labeled) samples from the training set to the \metaVal{}. As seen, allocating between $10-50$ samples results in similar performance picking at 20. The figure shows that removing too many samples from the training set can be damaging. Nevertheless, we notice that even when allocating 200 labeled samples (out of 1000), our nonlinear method is still better than the best competitor GSC \citep{du2018adapting} (which reached an accuracy of $42.57$).

\figref{fig:cub_num_layers_effect} 
shows how accuracy changes with the number of hidden layers. As expected, there is a positive trend. As we increase the number of layers, the network expressivity increases, and the performance improves. Clearly, making the \GuideNet{} too large may cause the network to overfit the \metaVal{} as was shown in~\secref{sec:analysis}, and empirically in \citep{lorraine2019optimizing}. 

\begin{figure}[t]
\centering
\subfloat[Effect of \metaVal{} size]{
\includegraphics[width=0.45\linewidth]{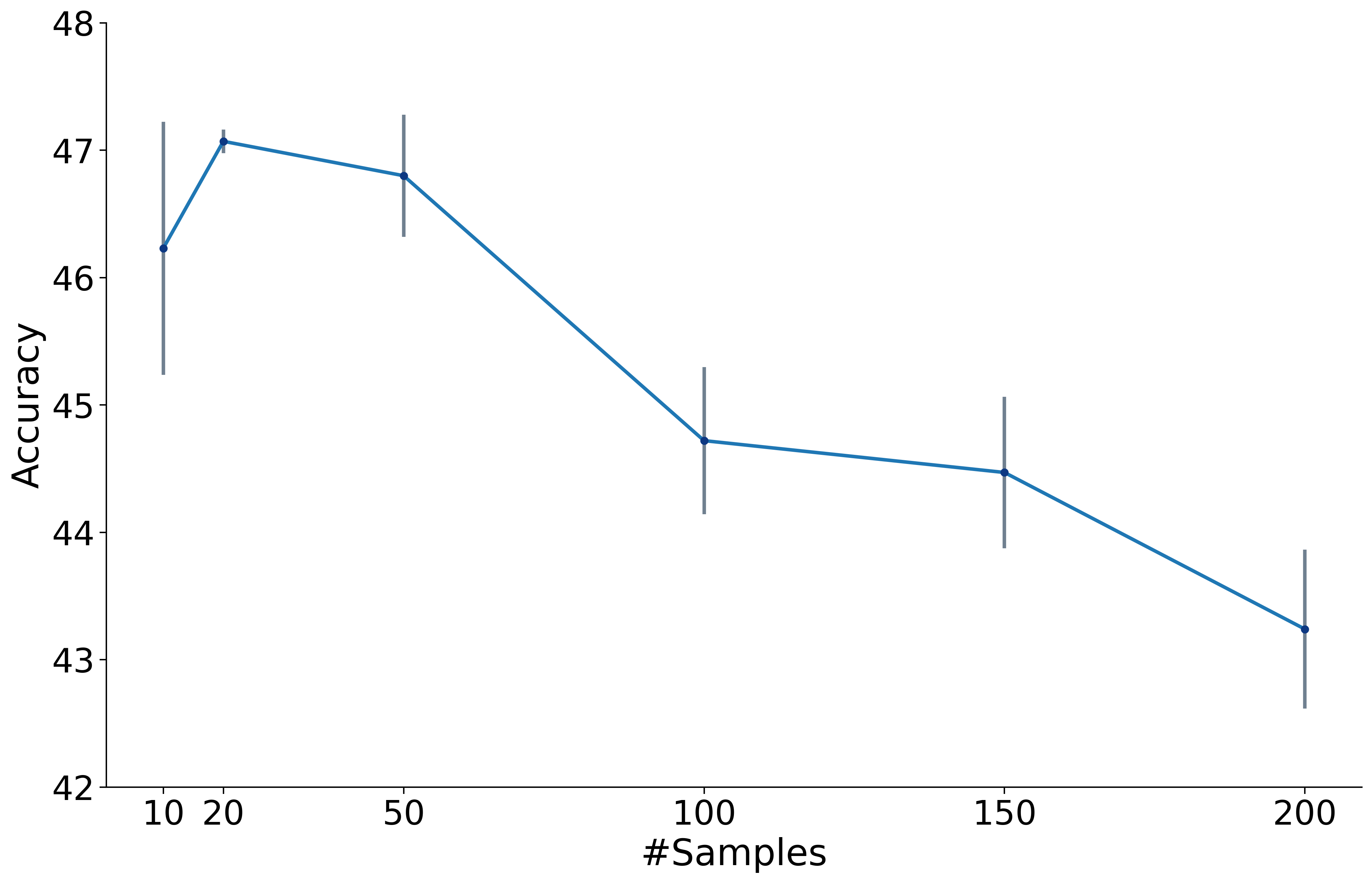}
\label{fig:cub_meta-val_effect}}
\subfloat[Effect of Depth]{
\includegraphics[width=0.45\linewidth]{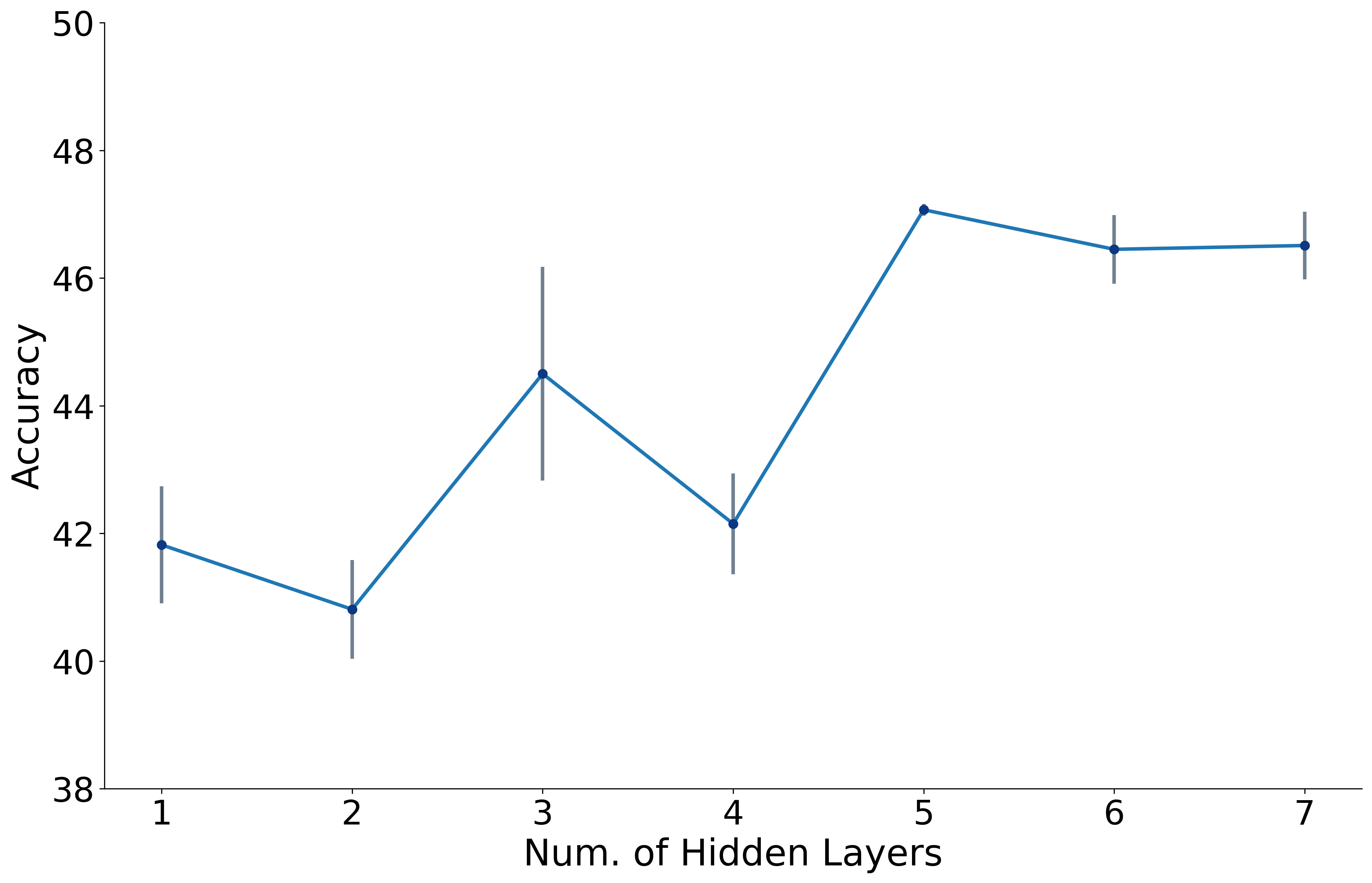}
\label{fig:cub_num_layers_effect}}
\caption{Mean test accuracy ($\pm$ SEM) averaged over 3 runs as a function of the number of samples in the \metaVal{} (left) and the number of hidden layers (right). Results are on 5-shot CUB 200-2011 dataset.}
\label{fig:cub_choices}
\end{figure}

\subsection{Linearly weighted non-linear terms} \label{sec:linear_weighted_poly}
To further motivate the use of non-linear interactions between tasks, we train a linear \GuideNet{} over a polynomial kernel on the tasks segmentation, depth estimation and normal prediction from the NYUv2 dataset. \figref{fig:nyu_poly} shows the learned loss weights. From the figure, we learn that two of the three largest weights at the end of training belong to non-linear terms, specifically, $Seg^2$ and $Seg\cdot Depth$. Also, we observe a \emph{scheduling} effect, in which at the start of training, the \GuideNet{} focuses on the auxiliary tasks (first $\sim50$ steps), and afterwards it draws most of the attention of the primary network towards the main task. 

\begin{figure}[t]
    \centering
    \text{Polynomial kernel - linear weights}
    \includegraphics[width=.8\linewidth]{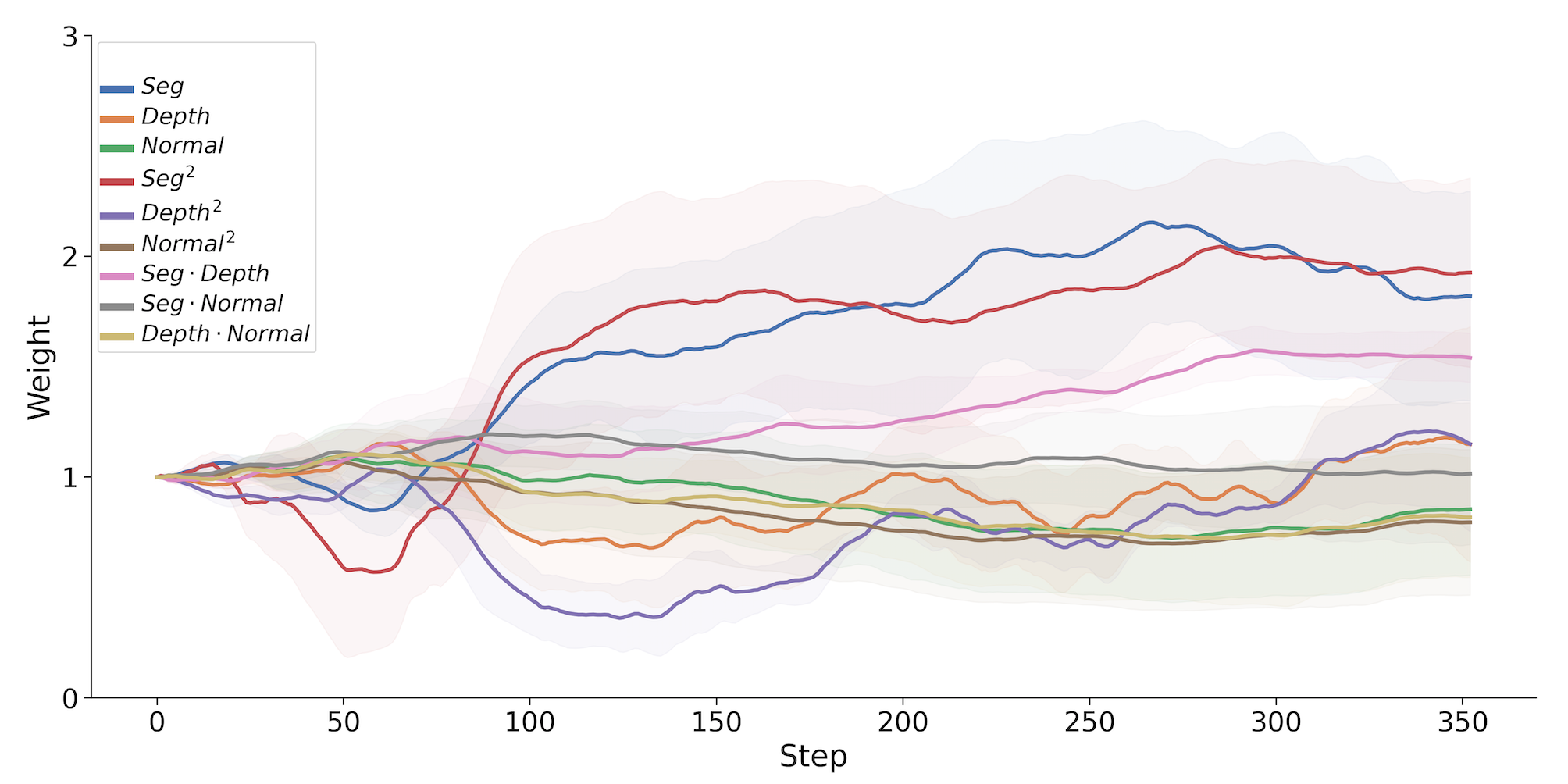}
    \caption{Learned linear weights for a polynomial kernel on the loss terms of the tasks segmentation, depth estimation and normal prediction from the NYUv2 dataset.}
    \label{fig:nyu_poly}
\end{figure}

\subsection{Fixed auxiliary}
\label{sec:fixed_aux_params}
As a result of alternating between optimizing the primary network parameters and the auxiliary parameters, the weighting of the loss terms are updated during the training process. This means that the loss landscape is changed during training. This effect is observed in the illustrative examples described in \secref{sec:illustrative} and \secref{sec:linear_weighted_poly}, where the \GuideNet{} focuses on different tasks during different learning stages. Since the optimization is non-convex, the end result may depend not only on the final parameters but also on the loss landscape during the entire process.

We examined this effect with the following setup on the 5-shot setting on CUB 200-2011 dataset: we trained a non-linear \GuideNet{} and saved the best model. Then we retrain with the same configuration, only this time, the \GuideNet{} is initialized using the best model, and is kept fixed. We repeat this using ten different random seeds, affecting the primary network initialization and data shuffling. As a result, we observed a drop of 6.7\% on average in the model performance with an std of 1.2\% (46.7\% compared to 40\%).


\subsection{Full CUB dataset}
\label{sec:cub_full_data}
In \secref{sec:cub_exp} we evaluated \AuxiLearn{} and the baseline models performance under a semi-supervised scenario in which we have $5$ or $10$ labeled samples per class. For completeness sake, we show in Table \ref{cub_tab:results_all} the test accuracy results in the standard fully-supervised scenario. As can be seen, in this case the STL baseline achieves the highest top-1 test accuracy while our nonlinear method is second on the top-1 and first on the top-3. Most baselines suffer from severe negative transfer due to the large number of auxiliary tasks (which are not needed in this case) while our method cause minimal performance degradation.

\begin{table}[!h]
\caption{CUB 200-2011: Fully supervised test classification accuracy ($\pm$ SEM) averaged over three runs. \newline}
\centering
\scalebox{0.85}{
    \begin{tabular}{l c c}
    \toprule
     & Top 1 & Top 3\\ 
    \midrule
    STL & \textbf{75.2 $\pm$ 0.52} & 88.4 $\pm$ 0.36 \\
    Equal & 70.16 $\pm$ 0.10 & 86.87 $\pm$ 0.22 \\
    Uncertainty & 74.70 $\pm$ 0.56 & 88.21 $\pm$ 0.14 \\
    DWA & 69.88 $\pm$ 0.10 & 86.62 $\pm$ 0.20 \\
    GradNorm & 70.04 $\pm$ 0.21 & 86.63 $\pm$ 0.13 \\
    GSC & 71.30 $\pm$ 0.01 & 86.91 $\pm$ 0.28 \\
    \midrule
    \textbf{\AuxiLearn{} (ours)} \\
    Linear & 70.97$\pm$ 0.31 & 86.92 $\pm$ 0.08 \\
    Deep Linear & 73.6 $\pm$ 0.72 & 88.37 $\pm$ 0.21 \\
    Nonlinear & 74.92 $\pm$ 0.21 & \textbf{88.55 $\pm$ 0.17} \\
    
    \bottomrule
    \end{tabular}}
\label{cub_tab:results_all}
\end{table}

\subsection{Cityscapes}
\label{sec:cityscapes}

Cityscapes~\citep{cordts2016cityscapes} is a high-quality urban-scene dataset. We use the data provided in~\citep{liu2019end} with 2975 training and 500 test images. The data comprises of four learning tasks: 19-classes, 7-classes and 2-classes semantic segmentation, and depth estimation. We use the 19-classes semantic segmentation as the main task, and all other tasks as auxiliaries. We allocate $10\%$ of the training data for validation set, to allow for hyperparameter tuning and early stopping. We further allocate $2.5\%$ of the remaining training examples to construct the \metaVal{}. All images are resized to $128 \times 256$ to speed up computation. 

We train a SegNet~\citep{badrinarayanan2017segnet} based model for $150$ epochs using Adam optimizer~\citep{Kingma2014AdamAM} with learning rate $1e-4$, and halve the learning rate after $100$ epochs. We search over weight decay in $\{1e-4, 1e-5\}$. We compare \AuxiLearn{} to the same baselines used in \secref{sec:cub_exp} and search over the same hyperparameters as in the NYUv2 experiment. We set the DWA temperature to $2$ similar to~\citep{liu2019end}, and the GradNorm hyperparameter $\alpha$ to $1.5$, as used in~\citep{chen2017gradnorm} for the NYUv2 experiments. We present the results in Table~\ref{Tab:cityscapes}. The ConvNet variant of the \GuideNet{} achieves best performance in terms of mIoU and pixel accuracy.

\begin{table}[h]
\centering
\caption{19-classes semantic segmentation test set results on Cityscapes, averaged over three runs ($\pm$ SEM).
}
\scalebox{0.8}{
    \begin{tabular}[h]{l  c  c}
    \toprule
    &mIoU &Pixel acc.\\
    \midrule
    STL & $30.18 \pm 0.04$ & $87.08 \pm 0.18$\\
    Equal & $30.45 \pm 0.14$ & $87.14 \pm 0.08$\\
    Uncertainty & $30.49 \pm 0.21$ & $86.89 \pm 0.07$\\
    DWA & $30.79 \pm 0.32$ & $86.97 \pm 0.26$\\
    GradNorm & $30.62 \pm 0.03$ & $87.15 \pm 0.04$\\
    GCS & $30.32 \pm 0.23$ & $87.02 \pm 0.12$\\
    \midrule
    \textbf{\AuxiLearn{} (ours)} & & \\
    Linear & $30.63 \pm 0.19$ & $86.88 \pm 0.03$\\
    Nonlinear & $30.85 \pm 0.19$ &
    $87.19 \pm 0.20$\\
    ConvNet & $\mathbf{30.99 \pm 0.05}$ & $\mathbf{87.21 \pm 0.11}$\\
    \bottomrule
    \end{tabular}}
\label{Tab:cityscapes}
\end{table}

\subsection{Learning segmentation auxiliary for 3D point clouds}
\label{sec:pc}
Recently, several methods were offered for learning auxiliary tasks in point clouds \citep{achituve2020self, hassani2019unsupervised, sauder2019self}; however, this domain is still largely unexplored and it is not yet clear which auxiliary tasks could be beneficial beforehand. Therefore, it is desirable to automate this process, even at the cost of performance degradation to some extent compared to human designed methods. 

We further evaluate our method in the task of generating helpful auxiliary tasks for 3D point-cloud data. We propose to extend the use of \AuxiLearn{} for segmentation tasks. In \secref{sec:learning_cls_exp} we trained an \GuideNet{} to output soft auxiliary labels for classification task. Here, we use a similar approach, assigning a soft label vector to each point. We then train the primary network on the main task and the auxiliary task of segmenting each point based on the learned labels.



We evaluated the above approach in a part-segmentation task using the ShapeNet part dataset \citep{yi2016scalable}. This dataset contains 16,881 3D shapes from 16 object categories (including Airplane, Bag, Lamp), annotated with a total of 50 parts (at most 6 parts per object). The main task is to predict a part label for each point. We follow the official train/val/test split scheme in \citep{chang2015shapenet}. We also follow the standard experimental setup in the literature, which assumes known object category labels during segmentation of a shape (see e.g., \citep{qi2017pointnet, wang2019dynamic}). During training we uniformly sample 1024 points from each shape and we ignore points normal. During evaluation we use all points of a shape. For all methods (ours and baselines) we used the DGCNN architecture \citep{wang2019dynamic} as the backbone feature extractor and for part segmentation. We evaluated performance using point-Intersection over Union (IoU) following \citep{qi2017pointnet}.

We compared AuxiLearn with the following baselines: \textbf{(1) Single Task Learning (STL):} Training with the main task only. \textbf{(2) DefRec:} An auxiliary task of reconstructing a shape with a deformed region \citep{achituve2020self}. \textbf{(3) Reconstructing Spaces (RS):} An auxiliary task of reconstructing a shape from a shuffled version of it \citep{sauder2019self}. and \textbf{(4) Denoising Auto-encoder (DAE):} An auxiliary task of reconstructing a point-cloud perturbed with an iid noise from $\mathcal{N}(0,0.01)$.

We performed hyper-parameter search over the primary network learning rate in $\{1e-3, 1e-4\}$, weight decay in $\{5e-5, 1e-5\}$ and weight ratio between the main and auxiliary task of $\{1:1, 1:0.5, 1:0.25\}$. We trained each method for 150 epochs, used the Adam optimizer with cosine scheduler. We applied early stopping based on the mean IoU of the validation set. We ran each configuration with 3 different seeds and report the average mean IOU along with the SEM. 
We used the segmentation network proposed in \citep{wang2019dynamic} with an exception that the network wasn't supplied with the object label as input.

\begin{table}[!t]
\centering
\tiny
\caption{Learning auxiliary segmentation task. Test mean IOU on ShapeNet part dataset averaged over three runs ($\pm$SEM) - 30 shot}
\setlength\tabcolsep{2.2pt}
\begin{tabular}{l c ccccccccccccccccc}
\toprule
& Mean && Airplane & Bag & Cap & Car & Chair & Earphone & Guitar & Knife & Lamp & Laptop & Motorbike & Mug & Pistol & Rocket & Skateboard & Table \\
\cmidrule{1-2}\cmidrule{4-19}
Num. samples & 2874 && 341 & 14 & 11 & 158 & 704 & 14 & 159 & 80 & 286 & 83 & 51 & 38 & 44 & 12 & 31 & 848 \\
\cmidrule{1-2}\cmidrule{4-19}
STL & 75.6 && 68.7 & \textbf{82.9} & 85.2 & \textbf{65.6} & 82.3 & 70.2 & 86.1 & 75.1 & 68.4 & 94.3 & 55.1 & 91.0 & 72.6 & 60.2 & 72.3 & 74.2 \\
DAE & 74.0 && 66.6 & 77.6 & 79.1 & 60.5 & 81.2 & \textbf{73.8} & 87.1 & 77.0 & 65.4 & 93.6 & 51.8 & 88.4 & \textbf{74.0} & 55.4 & 68.4 & 72.7 \\
DefRec  & 74.6 && 68.6 & 81.2 & 83.8 & 63.6 & 82.1 & 72.9 & 86.9 & 72.7 & 69.4 & 93.4 & 51.8 & 89.7 & 72.0 & 57.2 & 70.5 & 71.7 \\
RS  & \textbf{76.5} && \textbf{69.7} & 79.1 & \textbf{85.9} & 64.9 & \textbf{83.8} & 68.4 & 82.8 & 79.4 & \textbf{70.7} & 94.5 & \textbf{58.9} & 91.8 & 72.0 & 53.4 & 70.3 & \textbf{75.0} \\
\cmidrule{1-2}\cmidrule{4-19}
AuxiLearn & 76.2 && 68.9 & 78.3 & 83.6 & 64.9 & 83.4 & 69.7 & \textbf{87.4} & \textbf{80.7} & 68.3 & \textbf{94.6} & 53.2 & \textbf{92.1} & 73.7 & \textbf{61.6} & \textbf{72.4} & 74.6 \\
\bottomrule
\end{tabular}
\label{table:pc_seg}
\end{table}

For \AuxiLearn{}, we used a smaller version of PointNet \citep{qi2017pointnet} as the auxiliary network without input and feature transform layers. We selected PointNet because its model complexity is light and therefore is a good fit in our case. We learned a different auxiliary task per each object category (with 6 classes per category) since it showed better results. We performed hyper-parameter search over the auxiliary network learning rate in $\{1e-2, 1e-3\}$, weight decay in $\{5e-3, 5e-4\}$. Two training samples from each class were allocated for the \metaVal. 

Table \ref{table:pc_seg} shows the mean IOU per category when training with only 30 segmented point-clouds per object category (total of 480). As can be seen, \AuxiLearn{} performance is close to RS \citep{sauder2019self} and improve upon other baselines. This shows that in this case, our method generates useful auxiliary tasks that has shown similar or better gain than those designed by humans.

\subsection{Learning an auxiliary classifier} \label{sec:learning_aux_furter}

\begin{table}[h]
\parbox{.45\linewidth}{
\setlength{\tabcolsep}{3pt}
\scriptsize
\centering

\caption{Learning auxiliary task. Test accuracy averaged over three runs ($\pm$SEM) - 15 shot}

\begin{tabular}[!h]{l cc  cc}
    \toprule
    &\multicolumn{1}{c}{CUB} & &\multicolumn{1}{c}{Pet}\\
    \midrule
    STL & $22.6 \pm 0.2$ & &  $13.6 \pm 0.7$ \\
    MAXL-F & $24.2 \pm 0.7$ & & $14.1 \pm 0.1$  \\
    MAXL & $24.2 \pm 0.8$ & & $14.2 \pm 0.2$ \\
    \midrule
    \textbf{\AuxiLearn{}} & $\mathbf{26.1 \pm 0.7}$
    && $\mathbf{18.0 \pm 0.9}$ \\
    \bottomrule
\end{tabular}
\label{Tab:learnable_task_15shot}
}
\hfill
\parbox{.5\linewidth}{
    \centering
    \scriptsize
    \setlength{\tabcolsep}{3pt}
    \caption{
    CIFAR10 test results accuracy averaged over three runs ($\pm$SEM).}
    \centering
    \begin{tabular}{l cc cc cc}
    \toprule
     & \multicolumn{3}{c}{
     CIFAR10 
     } \\
     \cmidrule{2-4}\\
    & $10\%$ & $15\%$ & $100\%$ \\
    \midrule
    STL & $72.63 \pm 2.14$ & $80.30 \pm 0.09$ & $93.36 \pm 0.05$\\
    MAXL & $75.85 \pm 0.32$ & $81.37 \pm 0.26$ & $93.49 \pm 0.02$ \\
    \midrule
    \textbf{AuxiLearn} & $\mathbf{76.75 \pm 0.08}$ & $\mathbf{81.42 \pm 0.30}$ & $\mathbf{93.54 \pm 0.05}$ \\
    \bottomrule
    \end{tabular}
    \label{tab:cifar_extra}
}
\end{table}




In Section~\ref{sec:learning_cls_exp} we show how AuxiLearn learns a novel auxiliary to improve upon baseline methods. For the fine-grained classification experiments, we use only $30$ samples per class. Here we also compare AuxiLearn with the baseline methods when there are only $15$ images per class. Table ~\ref{Tab:learnable_task_15shot} shows that AuxiLearn is superior to baseline methods in this setup as well, even though it requires to allocate some samples from the training data to the auxiliary set.

\iclrrevision{To further examine the effect of learning novel auxiliary task with varying train set size, we provide here additional experiments on the CIFAR10 dataset. We evaluate the methods with of 10\%, 15\% and 100\% of training examples. The results are presented in Table~\ref{tab:cifar_extra}. As expected, learning with auxiliaries is mostly helpful in the low data regime. Nonetheless, AuxiLearn improves over single task learning and MAXL for all training set sizes.}

\section{Theoretical Considerations} \label{theoretical_cons}

In this section, we discuss the theoretical limitations of AuxiLearn. First, we discuss the smoothness of our loss criterion while learning to combine losses using DNNs. Next,
we present limitations that may arise from utilizing the IFT and their resolution. Finally, we discuss the approximations made for achieving an efficient optimization procedure.

\paragraph{Smoothness of the loss criterion.} When learning to combine losses as described in Section~\ref{sec:combine_losses}, one must take into consideration the smoothness of the learn loss criterion as a function of $W$. This limits, at least in theory, the design choice of the auxiliary network. In our experiments we use smooth activation functions, namely Softplus, to ensure the existence of $\partial{\mathcal{L}_T}/\partial{W}$. Nonetheless, using non-smooth activation (e.g. ReLU) results with a piecewise smooth loss function hence might work well in practice.

\paragraph{Assumptions for IFT.} One assumption for applying the IFT as described in Section~\ref{sec:unified_opt}, is that $\mathcal{L}_T$ is continuously differentiable w.r.t to the auxiliary and primary parameters. This assumption limits the design choice of both the auxiliary, and the primary networks. For instance, one must utilize only smooth activation functions. However, many non-smooth components can be replaced with smooth counterparts. For example, ReLU can be replaced with Softplus, $ReLU(x)=\lim_{\alpha\to\infty}\ln{(1+\exp(\alpha x))/\alpha}$,
and the beneficial effects of Batch-Normalization can be captured with Weight-Normalization as argued in~\citep{salimans2016weight}.

For the setup of \textit{learning to combine losses}, we use the above substitutes, namely Softplus and Weight Normalization, however for the \textit{learning a novel auxiliary} setup, we share architecture between primary and auxiliary network (e.g. ResNet18). While using non-smooth components may, in theory, cause issues, we show empirically through extensive experiment that AuxiLean performs well in practice, and its optimization is stable. Furthermore, we note that while ReLUs are non-smooth, they are piecewise smooth, hence the set of non-smoothness points is a zero-measure set.

\paragraph{Approximations.} Our optimization procedure relies on several approximations to efficiently solve complex bi-level optimization. This trade-off between computation efficiency and accurate approximation can be controlled by (i) The number of Neumann series components, and; (ii) The number of optimization steps between auxiliary parameters update. While we cannot guarantee that the bi-level optimization process converges, empirically we observe a stable optimization process.

Our work builds on previous studies in the field of hyperparameter optimization~\citep{lorraine2019optimizing,pedregosa2016hyperparameter}. \cite{lorraine2019optimizing} provide an error analysis for both approximations, in a setup for which the exact Hessian can be evaluated in closed form. We refer the readers to~\cite{pedregosa2016hyperparameter} for theoretical analysis and results regarding the second approximation (i.e. sub-optimally of the inner optimization problem in Eq.~\ref{eq:bi_level}).

\end{document}